\documentclass[11pt]{article}
\pdfoutput=1
\usepackage{fullpage}
\usepackage{times}
\usepackage{graphicx,color}
\usepackage{array,float}
\usepackage{url}
\usepackage[usenames,dvipsnames]{xcolor}
\usepackage{amstext,amssymb,amsmath}
\usepackage{amsthm}
\usepackage{verbatim}
\usepackage{bm}
\usepackage{paralist}
\usepackage{ulem}
\usepackage[numbers]{natbib}
\usepackage{todonotes}
\usepackage{hyperref}
\usepackage[noend]{algorithmic}
\usepackage{algorithm}

\newtheorem{lem}{Lemma}[section]

\newtheorem{thm}[lem]{Theorem}

\newtheorem{claim}[lem]{Claim}
\newcommand{\ip}[2]{\langle #1,#2\rangle}

\newcommand{\nptheta}{\hat\theta}

\newcommand{\privtheta}{\theta^{priv}}
\renewcommand{\paragraph}[1]{\vspace{3pt}\noindent\textbf{#1}}
\newcommand{\scrX}{\ensuremath{\mathcal{X}}}

\newcommand{\eL}{\mathcal{L}}

\newcommand{\good}{\sf Good}

\newcommand{\ltwo}[1]{\|#1\|_2}
\newcommand{\linf}[1]{\|#1\|_{\infty}}
\newcommand{\eps}{\epsilon}
\newcommand{\A}{\mathcal{A}}
\newcommand{\D}{\mathcal{D}}
\newcommand{\G}{\mathcal{G}}
\newcommand{\T}{\mathcal{T}}
\newcommand{\risk}{{\sf R}}
\newcommand{\vol}{{\sf Vol}}

\newcommand{\I}{\mathbb{I}}
\newcommand{\Ico}{\mathcal{I}_{\sf out}}
\newcommand{\Ici}{\mathcal{I}_{\sf in}}
\newcommand{\E}{\mathbb{E}}

\newcommand{\empL}{\mathcal{L}}

\newcommand{\dist}{{\sf Dist}_{\infty}}
\newcommand{\htheta}{\widetilde\theta}

\newcommand{\re}{\mathbb{R}}
\newcommand{\Bc}{\mathcal{B}}
\newcommand{\B}{\mathbb{B}}
\newcommand{\Q}{\mathcal{Q}}

\newcommand{\teps}{\tilde{\epsilon}}
\newcommand{\hmu}{\hat{\mu}}
\newcommand{\hmuA}{\hat{\mu}_{\sf A}}
\newcommand{\muA}{\mu_{\sf A}}
\newcommand{\hmuC}{\hat{\mu}_{\C}}
\newcommand{\muC}{\mu_{\C}}
\newcommand{\hmug}{\hat{\mu}_{\good}}

\newcommand{\mU}{\mathcal{U}}
\newcommand{\grad}{\bigtriangledown}

\newcommand{\hypcnz}{ \{-\frac{1}{\sqrt{p}}, \frac{1}{\sqrt{p}}\} }
\newcommand{\hypcnzsc}{ \{-\frac{\beta}{\sqrt{p}}, \frac{\beta}{\sqrt{p}}\} }

\newcommand{\mypar}[1]{\smallskip
\noindent{\bf\em {#1}:}}
\newcommand{\boldpar}[1]{\smallskip
\noindent{\bf{#1}:}}

\newcommand{\mynote}[3]{}

\newcommand{\asnote}[1]{\mynote{Adam}{yellow}{#1}}
\newcommand{\rbnote}[1]{\mynote{Raef}{green}{#1}}
\newcommand{\ignore}[1]{}

\newcommand{\C}{\mathcal{C}}

\newcommand{\erisk}{{\sf ExcessRisk}}
\newcommand{\emprisk}{{\sf EmpExcRisk}}
\newcommand{\paren}[1]{{\left ( {#1} \right)}}

\begin{document}

\title{Differentially Private Empirical Risk Minimization: Efficient Algorithms and Tight Error Bounds
}
\author{Raef Bassily\thanks{Computer Science and Engineering
    Department, The Pennsylvania State
    University. \texttt{\{bassily,asmith\}@psu.edu}. R.B. and
    A.S. were supported in part by NSF awards \#0747294 and \#0941553.}
\and Adam Smith\footnotemark[1]\ {~}\thanks{A.S. is on sabbatical at, and partly supported by, Boston University's Hariri
  Institute for Computing and Center for RISCS as well as Harvard University's
  Center for Computation and Society, via a Simons Investigator grant to Salil Vadhan.}
\and Abhradeep Thakurta\thanks{Stanford University and Microsoft
  Research. \texttt{b-abhrag@microsoft.com}. Supported in part by the Sloan Foundation.}
}

\maketitle
\thispagestyle{empty}

\begin{abstract}
In this paper, we initiate a systematic investigation of differentially private algorithms for convex empirical risk minimization. Various instantiations of this problem have been studied before.  We provide new algorithms and matching lower bounds for private ERM assuming only that each data point's contribution to the loss function is Lipschitz bounded and that the domain of optimization is bounded. We provide a separate set of algorithms and matching lower bounds for the setting in which the loss functions are known to also be strongly convex.

Our algorithms run in polynomial time, and in some cases even match the optimal nonprivate running time (as measured by oracle complexity). We give separate algorithms (and lower bounds) for $(\eps,0)$- and $(\eps,\delta)$-differential privacy; perhaps surprisingly, the techniques used for designing optimal algorithms in the two cases are completely different.

Our lower bounds apply even to very simple, smooth function families, such as linear and quadratic functions.  This implies that algorithms from previous work can be used to obtain optimal error rates, under the additional assumption that the contributions of each data point to the loss function is \emph{smooth}. We show that simple approaches to smoothing arbitrary loss functions (in order to apply previous techniques) do not yield optimal error rates. In particular, optimal algorithms were not previously known for problems such as training support vector machines and the high-dimensional median.
\end{abstract}


\newpage
\thispagestyle{empty}

\tableofcontents

\newpage
\pagenumbering{arabic}

\section{Introduction}
\label{sec:intro}

Convex optimization is one of the most basic and powerful computational tools in statistics and machine learning. It is most commonly used for empirical risk minimization (ERM): the data set $\D=\{d_1,...,d_n\}$ defines a convex loss function $\eL(\cdot)$ which is minimized over a convex set $\C$. When run on sensitive data, however, the results of convex ERM can leak sensitive information. For example, medians and support vector machine parameters can, in many cases, leak entire records in the clear (see ``Motivation'', below).

In this paper, we provide new algorithms and matching lower bounds for \emph{differentially private} convex ERM assuming only that each data point's contribution to the loss function is Lipschitz and that the domain of optimization is bounded. This builds on a line of work started by Chaudhuri et al.~\cite{CMS11}.

\newcommand{\universe}{\scrX}

\paragraph{Problem formulation.}
Given a data set $\D=\{d_1,...,d_n\}$ drawn from a universe $\universe$, and a closed, convex set $\C$, our goal is to
$$\text{minimize }\eL(\theta ; \D) = \sum_{i=1}^n \ell(\theta; d_i) \text{ over }\theta \in \C$$
The map $\ell$ defines, for each data point  $d$, a loss function
$\ell(\cdot; d)$ on $\C$. We will generally assume that $\ell(\cdot;d)$ is convex and $L$-Lipschitz for all $d\in\universe$.
One obtains variants on this basic problem by assuming additional restrictions, such as (i) that $\ell(\cdot;d)$ is $\Delta$-strongly convex for all $d\in \universe$, and/or (ii) that $\ell(\cdot;d)$ is $\beta$-smooth for all $d\in \universe$. Definitions of Lipschitz, strong convexity and smoothness are provided at the end of the introduction.

For example, given a collection of data points in $\re^p$, the Euclidean 1-median is a point in $\re^p$ that minimizes the sum of the Euclidean distances to the data points. That is, $\ell(\theta; d_i)=\|\theta-d_i\|_2$, which is 1-Lipschitz in $\theta$ for any choice of $d_i$. Another common example is the support vector machine (SVM): given a data point $d_i=(x_i,y_i)\in \re^p\times\{-1,1\}$, one defines a loss function $\ell(\theta;d_i) = hinge(y_i\cdot \langle\theta, x_i\rangle)$, where $hinge(z) = (1-z)_+$ (here $(1-z)_+$ equals $1-z$ for $z\leq 1$ and 0, otherwise). The loss is $L$-Lipshitz in $\theta$ when $\|x_i\|_2\leq L$.

Our formulation also captures \emph{regularized} ERM, in which an additional (convex) function $r(\theta)$ is added to the loss function to penalize certain types of solutions; the loss function is then $r(\theta)+\sum_{i=1}^n\ell(\theta;d_i)$. One can fold the regularizer $r(\cdot)$ into the data-dependent functions by replacing $\ell(\theta;d_i)$ with $\tilde\ell(\theta;d_i) = \ell(\theta;d_i) +\frac 1 n r(\theta)$, so that $\eL(\theta;\D) = \sum_i \tilde\ell(\theta;d_i)$. This folding comes at some loss of generality (since it may increase the Lipschitz constant), but it does not affect asymptotic results. Note that if $r$ is $\Delta n$-strongly convex, then every $\tilde{\ell}$ is $\Delta$-strongly convex.

\rbnote{can we say below that emp ex. risl lower bounds are also lower bounds on gen. error?}
\rbnote{can we talk below about cases we know to be tight? e.g., GLM. In App F, I added a theorem statement for the upper bounds in this case.}

We measure the success of our algorithms by the worst-case (over inputs) expected \emph{excess empirical risk}, namely
\begin{equation}
\E(\eL(\hat\theta;\D) - \eL(\theta^*;\D)),\label{eq:excessRisk}
\end{equation}
where $\hat\theta$ is the output of the algorithm, $\theta^* = \arg\min_{\theta \in \C} \eL(\theta;\D)$ is the true minimizer, and the expectation is only over the coins of the algorithm. 
Expected risk guarantees can be converted to high-probability guarantees using standard amplification techniques (see Appendix \ref{sec:expectatioHigh} for details).

Another important measure of performance is an algorithm's (excess) generalization error, where loss is measured with respect to the average over an unknown distribution from which the data are assumed to be drawn i.i.d.. Our upper bounds on empirical risk imply upper bounds on generalization error (via uniform convergence and similar ideas); the resulting bounds are only known to be tight in certain ranges of parameters, however.  Detailed statements may be found in Appendix~\ref{app:generalization}. \asnote{Review this section and add what we know about lower bounds.}

\paragraph{Motivation.} Convex ERM is used for fitting models from simple least-squares regression to support vector machines, and their use may have significant implications to privacy.
As a simple example, note that the Euclidean 1-median of a data set will typically be an actual data point, since the gradient of the loss function has discontinuities at each of the $d_i$. (Thinking about the one-dimensional median, where there is \emph{always} a data point that minimizes the loss, is helpful.) Thus, releasing the median may well reveal one of the data points in the clear. A more subtle example is the support vector machine (SVM). The solution to an SVM program is often presented in its dual form, whose coefficients typically consist of a set of $p+1$ exact data points. \citet{KRS13} show how the results of many convex ERM problems can be combined to carry out reconstruction attacks in the spirit of \citet{DiNi03}.

\paragraph{Differential privacy} is a rigorous notion of privacy that emerged from a line of work in theoretical computer science and cryptography \cite{DwNi04,BDMN05,DMNS06}. We say two data sets $\D$ and $\D'$ of size $n$ are neighbors if they differ in one entry (that is, $\left\vert\D\triangle\D'\right\vert = 2$).  A randomized  algorithm $\A$ is $(\eps,\delta)$-differentially private (\citet{DMNS06,DKMMN06}) if, for all neighboring data sets $\D$ and $\D'$ and for all events $S$ in the output space of $\A$, we have
$$\Pr(\A(\D)\in S) \leq e^{\eps} \Pr(\A(\D')\in S) +\delta\,.$$
Algorithms that satisfy differential privacy for $\eps<1$ and $\delta \ll 1/n$ provide meaningful privacy guarantees, even in the presence of side information. In particular, they avoid the problems mentioned in ``Motivation'' above. See \citet{Dwork06,KS08,KiferM12} for discussion of the ``semantics'' of differential privacy.

\paragraph{Setting Parameters.}
We will aim to quantify the role of several basic parameters on the excess risk of differentially private algorithms: the size of the data set $n$, the dimension $p$ of the parameter space $\C$, the Lipschitz constant $L$ of the loss functions, the diameter $\ltwo{\C}$ of the constraint set and, when applicable, the strong convexity $\Delta$. 

We may take $L$ and $\ltwo{\C}$ to be 1 without loss of generality: We can set $\ltwo{\C}=1$ by 
rescaling $\theta$ (replacing by $\theta$ with $\theta \cdot  \ltwo{\C}$); we can then set $L=1$ by rescaling the loss function $\eL$ (replacing $\eL$ by $\eL/ L$). These two transformations change the excess risk by $L\ltwo{C}$. The parameter $\Delta$ cannot similarly be rescaled while keeping $L$ and $\ltwo{\C}$ the same. However, we always have $\Delta\leq 2L/\ltwo{\C}$. 

In the sequel, we thus focus on the setting where $L=\ltwo{\C}=1$ and $\Delta \in [0,2]$. 
To convert excess risk bounds for $L=\ltwo{\C}=1$ to the general setting, one can multiply the risk bounds by $L\ltwo{\C}$, and replace $\Delta$ by $\frac{\Delta \ltwo{\C}}{L}$.


\begin{table*}[t]
\hspace{-0.5in}
\begin{center} \small
  \begin{tabular}{|p{1.22in}||c|c|c||c|c|c|}
    \hline
    & \multicolumn{3}{c||}{$(\eps,0)$-DP}  & \multicolumn{3}{c|}{$(\eps,\delta)$-DP} \\
    \hline
& Previous \citep{CMS11} &\multicolumn{2}{c||}{\color{blue} This work} & Previous \citep{KST12} &\multicolumn{2}{c|}{\color{blue} This work} \\
\hline
    Assumptions  & Upper Bd & \color{blue} Upper Bd & \color{blue} Lower Bd & Upper Bd & \color{blue} Upper Bd & \color{blue} Lower Bd \\
    \hline
    \hline
    $1$-Lipschitz and $\ltwo{C}=1$&
    $\dfrac{p\sqrt{n}}{\eps}$ &
    \color{blue} $\dfrac{p}{\eps}$ &
    \color{blue} $\dfrac{p}{\eps}$ &
    $\dfrac{\sqrt{p\cdot n}\log(1/\delta)}{\eps}$ &
    \color{blue} $\dfrac{\sqrt{p}\log^{2}(n/\delta)}{\eps}$ &
    \color{blue} $\dfrac{\sqrt{p}}{\eps}$ \\
     \hline
   ... and $O(p)$-smooth
    & $\dfrac{p}{\eps}$  & &  \color{blue} $\dfrac{p}{\eps}$&      $\dfrac{\sqrt{p}\log(1/\delta)}{\eps}$ & & \color{blue} $\dfrac{\sqrt{p}}{\eps}$ \\
    \hline
    \hline
    $1$-Lipschitz and $\Delta$-strongly convex and $\ltwo{C}=1$
(implies $\Delta \leq 2$)&
    $\dfrac{p^2}{\sqrt{n}\Delta \eps^2}$ &
    \color{blue} $\dfrac{\log(n)}{\Delta}\cdot \dfrac{p^2}{n\eps^2}$ &
    \color{blue} $ \dfrac{p^2}{n\eps^2} $ &
    $\dfrac{p\log(1/\delta)}{\sqrt{n} \Delta \eps^2}$ &
    \color{blue} $\dfrac{\log^3(n/\delta)}{\Delta}\cdot \dfrac{p}{n \eps^2}$ &
    \color{blue} $\dfrac{p }{n\eps^2} $
\\
    \hline
   ... and $O(p)$-smooth
    & $\dfrac{p^2}{n\Delta \eps^2}$ &
    &
    \color{blue}
    $\dfrac{p^2}{n \eps^2}$  &
    $\dfrac{p\log(1/\delta)}{n\Delta \eps^2}$ & &
\color{blue} $\dfrac{p}{n \eps^2}$ \\
    \hline
  \end{tabular}
  \end{center}
  \caption{Upper and lower bounds for excess risk of differentially-private convex ERM. Bounds ignore leading multiplicative constants, and the values in the table give the bound when it is below $n$. That is, upper bounds should be read as $O(\min(n, ...))$ and lower bounds, as $\Omega(\min(n,...))$). Here $\ltwo{\C}$ is the diameter of $\C$. The bounds are stated for the setting where $L=\ltwo{\C}=1$, which can be enforced by rescaling; to get general statements, multiply the risk bounds by $L\ltwo{\C}$, and replace $\Delta$ by $\frac{\Delta \ltwo{\C}}{L}$. We assume $\delta<1/n$ to simplify the bounds.}
  \label{tab:bounds}
\end{table*}

\subsection{Contributions}
\label{sec:contribs}

We give algorithms that significantly improve on the state of the art for optimizing non-smooth loss functions --- for both the general case and strongly convex functions, we improve the excess risk bounds by a factor of $\sqrt{n}$, asymptotically. The algorithms we give for $(\eps,0)$- and $(\eps,\delta)$-differential privacy work on very different principles. We group the algorithms below by technique: gradient descent, exponential sampling, and localization.

For the purposes of this section, $\tilde O(\cdot)$ notation hides factors polynomial in  $\log n$ and $\log(1/\delta)$. Detailed bounds are stated in Table \ref{tab:bounds}.

\paragraph{Gradient descent-based algorithms.} For $(\eps,\delta)$-differential privacy, we show that a noisy version of gradient descent achieves excess risk $\tilde O(\sqrt{p} /\eps)$. This matches our lower bound,  $\Omega (\min(n,\sqrt{p}/\eps))$, up to logarithmic factors. (Note that every $\theta \in \C$ has excess risk at most $n$, so a lower bound of $n$ can always be matched.)
For $\Delta$-strongly convex functions, a variant of our algorithm has risk $\tilde O(\frac{p}{\Delta  n \eps^2})$, which matches the lower bound $\Omega(\frac p {n\eps^2})$ when $\Delta$ is bounded below by a constant (recall that $\Delta\leq 2$ since $L=\ltwo{\C}=1$).

Previously, the best known risk bounds were $\Omega(\sqrt{p n} /\eps)$ for general convex functions and $\Omega(\frac{p}{\sqrt{n}\Delta \eps^2})$ for $\Delta$-strongly convex functions (achievable via several different techniques (\citet{CMS11,KST12,JKT12,DuchiJW13})).
Under the restriction that each data point's contribution to the loss function is sufficiently smooth, objective perturbation \citep{CMS11,KST12} also has risk $\tilde O(\sqrt{p} /\eps)$ (which is tight, since the lower bounds apply to smooth functions). However, smooth functions do not include important special cases such as medians and support vector machines. \citet{CMS11} suggest applying their technique to support vector machines by smoothing (``huberizing'') the loss function. We show in Appendix~\ref{sec:huberization} that this approach still yields expected excess risk  $\Omega(\sqrt{p n} /\eps)$.

Although straightforward noisy gradient descent would work well in our setting, 
we present a faster variant based on \emph{stochastic} gradient descent: At each step $t$, the algorithm samples a random point $d_i$ from the data set, computes a noisy version of $d_i$'s contribution to the gradient of $\eL$ at the current estimate $\tilde\theta_t$, and then uses that noisy measurement to update the parameter estimate. The algorithm is similar to algorithms that have appeared previously (\citet{WM10} first investigated gradient descent with noisy updates; stochastic variants were studied by \citet{JKT12,DuchiJW13,SongCS13}). The novelty of our analysis lies in taking advantage of the randomness in the choice of $d_i$ (following \citet{KLNRS08}) to run the algorithm for many steps without a significant cost to privacy. Running the algorithm for $T=n^2$ steps, gives the desired expected excess risk bound. Even nonprivate first-order algorithms---i.e., those based on gradient measurements---must learn information about the gradient at $\Omega(n^2)$ points to get risk bounds that are independent of $n$ (this follows from ``oracle complexity'' bounds showing that $1/\sqrt{T}$ convergence rate is optimal \cite{NY83,AgarwalBRW12}). Thus, the query complexity of our algorithm cannot be improved without using more information about the loss function, such as second derivatives.

The gradient descent approach does not, to our knowledge, allow one to get optimal excess risk bounds for $(\eps,0)$-differential privacy. The main obstacle is that ``strong composition'' of $(\eps,\delta)$-privacy \citet{DRV} appears necessary to allow a first-order method to run for sufficiently many steps.

\paragraph{Exponential Sampling-based Algorithms.} For $(\eps,0)$-differential privacy, we observe that a straightforward use of the exponential mechanism --- sampling from an appropriately-sized net of points in $\C$, where each point $\theta$ has probability proportional to $\exp(-\eps \eL(\theta;\D))$ --- has excess risk $\tilde O(p /\eps)$ on general Lipschitz functions, nearly matching the lower bound  of $\Omega(p/\eps)$. (The bound would not be optimal for $(\eps,\delta)$-privacy because it scales as $p$, not $\sqrt{p}$).
This mechanism is inefficient in general since it requires construction of a net and an appropriate sampling mechanism.

We give a polynomial time algorithm that achieves the optimal excess risk, namely $O(p/\epsilon)$. Note that the achieved excess risk does not have any logarithmic factors which
is shown to be the case using a ``peeling-'' type argument that is specific to convex functions. The idea of our algorithm is to sample efficiently from the continuous distribution on all points in $\C$ with density $\mathcal{P}(\theta) \propto e^{-\eps\eL(\theta)}$. Although the distribution we hope to sample from is log-concave, standard techniques do not work for our purposes: existing methods converge only in statistical difference, whereas we require a \emph{multiplicative} convergence guarantee to provide $(\eps,0)$-differential privacy. Previous solutions to this issue (\citet{HT09}) worked for  the uniform distribution, but not for general log-concave distributions.

The problem comes from the combination of an arbitrary convex set and an arbitrary (Lipschitz) loss function defining $\mathcal{P}$. We circumvent this issue by giving an algorithm that samples from an appropriately defined distribution $\tilde{\mathcal{P}}$ on a cube containing $\C$, such that (i) the algorithm outputs a point in $\C$ with constant probability, and (ii) $\tilde{\mathcal{P}}$, conditioned on sampling from $\C$, is within multiplicative distance $O(\eps)$ from the correct distribution. We use, as a subroutine, the random walk on grid points of the cube of \cite{applegate1991sampling}.
\rbnote{I slightly rephrased/corrected the above statement}

\paragraph{Localization: Optimal Algorithms for Strongly Convex Functions.}
The exponential-sampling-based technique discussed above does not take advantage of strong convexity of the loss function. We show, however, that a novel combination of two standard techniques---the exponential mechanism and  Laplace-noise-based output perturbation---does yield an optimal algorithm. \citet{CMS11} and \cite{Rubinstein09} showed that strongly convex functions have low-sensitivity minimizers, and hence that one can release the minimum of a strongly convex function with Laplace noise (with total Euclidean length about $\rho = \frac{p}{\Delta \eps n}$ if each loss function is $\Delta$-strongly convex). Simply using this first estimate as a candidate output does not yield optimal utility in general; instead it gives a risk bound of roughly $\frac{p}{\Delta \eps}$.

The main insight is that this first estimate defines us a small neighborhood $\C_0\subseteq \C$, of radius about $\rho$, that contains the true minimizer. Running the exponential mechanism in this small set improves the excess risk bound by a factor of about $\rho$ over running the same mechanism on all of $\C$. The final risk bound is then  $\tilde O(\rho \frac{p}{\eps n} )= \tilde O(\frac{p^2}{\Delta \eps^2 n})$, which matches the lower bound of $\Omega (\frac{p^2}{\eps^2 n})$ when $\Delta = \Omega(1)$.  This simple ``localization'' idea is not needed for $(\eps,\delta)$-privacy, since the gradient descent method can already take advantage of strong convexity to converge more quickly.

\paragraph{Lower Bounds.} We use techniques developed to bound the accuracy of releasing 1-way marginals (due to \citet{HT09} for $(\eps,0)-$ and \citet{BUV13} for $(\eps,\delta)$-privacy) to show that our algorithms have essentially optimal risk bounds. The instances that arise in our lower bounds are simple: the functions can be linear (or quadratic, for the case of strong convexity) and the constraint set $\C$ can be either the unit ball or the hypercube.  In particular, our lower bounds apply to special case of  smooth functions, demonstrating the optimality of objective perturbation \cite{CMS11,KST12} in that setting. The reduction to lower-bounds for 1-way marginals is not quite black-box; we exploit specific properties of the instances used by \citet{HT09,BUV13}.

Finally, we provide a much stronger lower bound on the utility of a specific algorithm, the Huberization-based algorithm proposed by \citet{CMS11} for support vector machines. In order to apply their algorithm to nonsmooth loss functions, they proposed smoothing the loss function by Huberization, and then running their algorithm (which requires smoothness for the privacy analysis) on the resulting, modified loss functions. We show that for any setting of the Huerization parameters, there are simple, one-dimensional nonsmooth loss functions for which the algorithm has error $\Omega(n)$. This bound justifies the effort we put into designing new algorithms for nonsmooth loss functions.

\paragraph{Generalization Error.} 
\asnote{Added a discussion of generalization error. Comments?}
In Appendix~\ref{app:generalization},
we discuss the implications of our results for \emph{generalization
  error}. Specifically, suppose that the data are drawn i.i.d. from a
distribution $\tau$, and 
let $\erisk (\theta)$ denote the expected loss of
$\theta$ on unseen data from $\tau$, that is $\erisk(\theta) = 
\E_{d\sim\tau}\left[\ell(\theta;d)\right]$. 

For an important class of loss functions, called
\emph{generalized linear models}, the straightforward application of
our algorithms gives generalization error $\tilde
O\paren{\min\paren{\frac{1}{\sqrt{n}},\frac{p^c}{\eps n}}}$ where $c=1$ for the case of
  $(\eps,0)$-differential privacy, and $c=2$ for
  $(\eps,\delta)$-differential privacy (assuming
  $\log(1/\delta)=O(\log n)$). This bound is tight: the
  $\frac{1}{\sqrt{n}}$ is necessary even in the nonprivate setting,
and the necessity of the $\frac{p^c}{\eps n}$ term follows from our
lower bounds on excess empirical risk (they are also lower bounds on
generalization error). 

For the case of general Lipschitz convex functions, a modification of
our algorithms gives excess risk $\tilde
O\paren{\frac{p^{c'}}{\sqrt{\eps n}}}$, where $c'=\frac 12$ for
$(\eps,0)$-differential privacy and $c'=\frac 1 4$ for $(\eps,\delta)$
differential privacy (that is, the generalization error bound is
roughly the square root of the corresponding empirical error
bound). The best known lower bound, however, is the same as for the
special case of generalized linear models. The bounds match when
$p\approx n$ (in which case no nontrivial generalization error is
possible). However, for smaller values of $p$ there remains a gap that
is polynomial in $p$. Closing the gap is an interesting open problem.

\subsection{Other Related Work}

In addition to the previous work mentioned above, we mention several closely related works. 
A rich line of work seeks to characterize the optimal error of differentially private algorithms for learning and optimization \citet{KLNRS08,beimel2014bounds,ChaudhuriH11,BeimelNS13-itcs,BeimelNS13-random}. In particular, our results on $(\eps,0)$-differential privacy imply nearly-tight bounds on the ``representation dimension'' \citet{BeimelNS13-random} of convex Lipschitz functions. \asnote{Add these implications (probably to Appendix)?}

\citet{JT14} gave dimension-independent expected excess risk bounds for the special case of ``generalized linear models'' with a strongly convex regularizer, assuming that $\C=\re^p$ (that is, unconstrained optimization). \citet{KST12,ST13sparse} considered parameter convergence for high-dimensional sparse regression (where $p\gg n$). The settings of those papers are orthogonal to ours, though a common generalization would be interesting.

Efficient implementations of the exponential mechanism over infinite domains were discussed by \citet{HT09}, \citet{ChaudhuriSS13} and \citet{KapralovT13}. The latter two works were specific to sampling (approximately) singular vectors of a matrix, and their techniques do not obviously apply here.

Differentially private convex learning in different models has also been studied: for example, \citet{JKT12,DuchiJW13,ST13online} study online optimization, \citet{JT13} study an interactive model tailored to high-dimensional kernel learning. Convex optimization techniques have also played an important role in the development of algorithms for ``simultaneous query release'' (e.g., the line of work emerging from \citet{HR10}). We do not know of a direct connection between those works and our setting.






\subsection{Additional Definitions}  For completeness, we state a few additional definitions related to convex sets and functions.
\begin{itemize}
\item
  $\ell:\C\to\re$ is $L$-Lipschitz (in the Euclidean norm) if, for all
  pairs $x,y\in \C$, we have $|\ell(x)-\ell(y)| \leq L\|x-y\|_2$.
  A subgradient of a convex $\ell$ function at $x$, denoted $\partial \ell (x)$,  is the set of vectors $z$ such that for all $y\in \C$, $\ell(y)\geq \ell(x) + \ip{z}{y-x}$.

\item   $\ell$ is \emph{$\Delta$-strongly convex} on $\C$ if, for all $x\in
  \C$, for all subgradients $z$ at $x$, and
  for all $y\in \C$, we have $\ell(y) \geq \ell(x) + \ip{z}{y-x} +
  \frac{\Delta}{2}\ltwo{y-x}^2$ (i.e., $\ell$ is bounded
  \emph{below} by a quadratic function tangent at $x$).
\item   $\ell$ is \emph{$\beta$-smooth} on $\C$ if, for all $x\in \C$, for all subgradients $z$ at $x$  and for all $y\in \C$, we have
  $\ell(y) \leq \ell(x) + \ip{z}{y-x} + \frac{\beta}{2}\ltwo{y-x}^2$
  (i.e., $\ell$ is bounded \emph{above} by a quadratic function
  tangent at $x$). Smoothness implies differentiability, so the subgradient at $x$ is unique.

\item Given a convex set $\C$, we denote its diameter by $\ltwo{\C}$. We denote the projection of any vector $\theta\in\re^p$ to the convex set $\C$ by $\Pi_\C(\theta)=\arg\min\limits_{x\in\C}\ltwo{\theta-x}$.

\end{itemize}

\subsection{Organization of this Paper}
Our upper bounds (efficient algorithms) are given in Sections~\ref{sec:gradDesc}, \ref{sec:privConvexUpper}, and \ref{sec:out-pert}, whereas our lower bounds are given in Section~\ref{sec:privConvexLower}. Namely, in Section~\ref{sec:gradDesc}, we give efficient construction for $(\epsilon, \delta)$-differentially private algorithms for general convex loss as well as Lipschitz strongly convex loss. In Section~\ref{sec:privConvexUpper}, we discuss a pure $\epsilon$-differentially private algorithm for general Lipschitz convex loss and outline an efficient construction for such algorithm. In Section~\ref{sec:out-pert}, we discuss our localization technique and show how to construct efficient pure $\epsilon$-differentially private algorithms for Lipschitz strongly convex loss. We derive our lower bound for general Lipschitz convex loss in Section~\ref{sec:lipschitzConvexLower} and our lower bound for Lipschitz strongly convex loss in Section~\ref{sec:lipschitzConvexLowerStrong}. In Section~\ref{sec:implement-eff-exp}, we discuss a generic construction of an efficient algorithm for sampling (with a multiplicative distance guarantee) from a logconcave distribution over an arbitrary convex bounded set. As a by-product of our generic construction, we give the details of the construction of our efficient $\epsilon$-differentially private algorithm from Section~\ref{sec:eff-pure-eps-alg}.

The appendices contain proof details and supplementary material: Appendix~\ref{sec:huberization} shows that smoothing a nonsmooth loss function in order to apply the objective perturbation technique of \citet{CMS11} can introduce significant additional error. Appendix~\ref{sec:localization-eps-delta} gives details on the application of localization in the setting of $(\eps,\delta)$-differential privacy. Appendix \ref{app:proof-lower-bounds-lem} provides additional details on the proofs of lower bounds. In Appendix~\ref{sec:expectatioHigh}, we explain standard modifications that allow our algorithms to give high probability guarantees instead of expected risk guarantees. Finally, in Appendix \ref{app:generalization} we discuss the how our algorithms can be adapted to provide guarantees on generalization error, rather than empirical error.

\section{Gradient Descent and Optimal $(\eps,\delta)$-differentially private Optimization}
\label{sec:gradDesc}

\asnote{Switch from normal noise to DJW randomizer.}
In this section we provide an algorithm $\A_{\sf Noise-GD}$ (Algorithm \ref{Algo:GradDesc}) for computing $\privtheta$  using a \emph{noisy stochastic variant} of the classic gradient descent algorithm from the optimization literature \citep{Boyd}. Our algorithm (and the utility analysis) was inspired by the approach of \citet{WM10} for logistic regression.

All the excess risk bounds \eqref{eq:excessRisk} in this section and the rest of this paper, are presented in expectation over the randomness of the algorithm. In Section \ref{sec:expectatioHigh} we provide a generic tool to translate the expectation bounds into high probability bound albeit at a loss of extra logarithmic factor in the inverse of the failure probability.

\mypar{Note(1)} The results in this section do \emph{not} require the loss function $\ell$ to be differentiable. Although we present Algorithm $\A_{\sf Noise-GD}$ (and its analysis) using the  gradient of the loss function $\ell(\theta;d)$ at $\theta$, the same guarantees hold if instead of the gradient, the algorithm is run with any sub-gradient of $\ell$ at $\theta$.

\mypar{Note(2)} Instead of using the stochastic variant in Algorithm \ref{Algo:GradDesc}, one can use the complete gradient (i.e., $\grad \empL(\theta;\D)$) in Step \ref{line:GD2} and still have the same utility guarantee as Theorem \ref{thm:abcd4}. However, the running time goes up by a factor of $n$.

\begin{algorithm}[htb]
	\caption{$\A_{\sf Noise-GD}$: Differentially Private Gradient Descent}
	\begin{algorithmic}[1]
		\REQUIRE Data set: $\D=\{d_1,\cdots,d_n\}$, loss function $\ell$ (with Lipschitz constant $L$), privacy parameters $(\epsilon,\delta)$, convex set $\C$, and the learning rate function $\eta:[n^2]\to\re$.
        \STATE  Set noise variance $\sigma^2\leftarrow \frac{32 L^2 n^2\log(n/\delta)\log(1/\delta)}{\epsilon^2}$.
        \STATE $\htheta_1:$ Choose any point from $\C$.
        \FOR{$t=1$ to $n^2-1$}
            {\STATE Pick $d\sim_{u}\D$ with replacement.\label{line:GD1}}
            {\STATE $\htheta_{t+1}=\Pi_\C\left(\htheta_t-\eta(t)\left[n\grad \ell(\htheta_t;d)+b_t\right]\right)$, where $b_t\sim\mathcal{N}\left(0,\mathbb{I}_{p}\sigma^2\right)$.\label{line:GD2}}
        \ENDFOR
        \STATE Output $\privtheta=\htheta_{n^2}$.
	\end{algorithmic}
	\label{Algo:GradDesc}
\end{algorithm}

\begin{thm}[Privacy guarantee]
Algorithm $\A_{\sf Noise-GD}$ (Algorithm \ref{Algo:GradDesc}) is $(\epsilon,\delta)$-differentially private.
\label{thm:priv13}
\end{thm}

\begin{proof}

\newcommand{\xtd}{{X_t(\D)}}
\newcommand{\xtdp}{{X_t(\D')}}

	At any time step $t\in[n^2]$ in Algorithm $\A_{\sf Noise-GD}$, fix the randomness due to sampling in Line \ref{line:GD1}. Let $\xtd=n\grad\ell(\htheta_t;d)+b_t$ be a random variable defined over the randomness of $b_t$ and conditioned on $\htheta_t$ (see Line \ref{line:GD2} for a definition), where $d\in\D$ is the data point picked in Line \ref{line:GD1}. Denote $\mu_\xtd(y)$ to be the measure of the random variable $X_t(\D)$ induced on $y\in\re$. For any two neighboring data sets $\D$ and $\D'$, define the \emph{privacy loss} random variable \citep{DRV} to be $W_t=\left|\log\frac{\mu_{\xtd}(X_t(\D))}{\mu_{\xtdp}(X_t(\D))}\right|$. Standard differential privacy arguments for Gaussian noise addition (see \cite{KST12,NTZ13}) will ensure that with probability $1-\frac{\delta}{2}$ (over the randomness of the random variables $b_t$'s and conditioned on the randomness due to sampling), $W_t\leq\frac{\epsilon}{2\sqrt{\log(1/\delta)}}$ for all $t\in[n^2]$. Now using the following lemma (Lemma \ref{claim:privAmp} with $\epsilon'=\frac{\epsilon}{2\sqrt{\log(1/\delta)}}$ and $\gamma=1/n$) we ensure that over the randomness of $b_t$'s and the randomness due to sampling in Line \ref{line:GD1} , w.p. at least $1-\frac{\delta}{2}$, $W_t\leq\frac{\epsilon}{n\sqrt{\log(1/\delta)}}$ for all $t\in[n^2]$. While using Lemma \ref{claim:privAmp}, we ensure that the condition $\frac{\epsilon}{2\sqrt{\log(1/\delta)}}\leq 1$ is satisfied.
	
	\begin{lem}[Privacy amplification via sampling. Lemma 4 in \cite{beimel2014bounds}]
		Over a domain of data sets $\T^n$, if an algorithm $\A$ is $\epsilon'\leq 1$ differentially private, then for any data set $\D\in\T^n$, executing $\A$ on uniformly random $\gamma n$ entries of $\D$ ensures $2\gamma\epsilon'$-differential privacy.
		\label{claim:privAmp}
	\end{lem}
	
	To conclude the proof, we apply ``strong composition'' (Lemma \ref{thm:strongComposition}) from \cite{DRV}. With probability at least $1-\delta$, the privacy loss $W=\sum\limits_{t=1}^{n^2} W_t$ is at most $\epsilon$. This concludes the proof.
	
	\begin{lem}[Strong composition \citep{DRV}]
		Let $\epsilon,\delta'\geq 0$. The class of $\epsilon$-differentially private algorithms satisfies $(\epsilon',\delta')$-differential privacy under $T$-fold adaptive composition for $\epsilon'=\sqrt{2T\ln(1/\delta')}\epsilon+T\epsilon(e^\epsilon-1)$.
		\label{thm:strongComposition}
	\end{lem}
\end{proof}
In the following we provide the utility guarantees for Algorithm $\A_{\sf Noise-GD}$ under two different settings, namely, when the function $\ell$ is Lipschitz, and when the function $\ell$ is Lipschitz and strongly convex. In Section \ref{sec:privConvexLower} we argue that these excess risk bounds are essentially tight.

\begin{thm}[Utility guarantee]
Let $\sigma^2=O\left(\frac{L^2 n^2\log(n/\delta)\log(1/\delta)}{\epsilon^2}\right)$. For  $\privtheta$ output by Algorithm $\A_{\sf Noise-GD}$ we have the following. (The expectation is over the randomness of the algorithm.)
\begin{enumerate}
\item \mypar{Lipschitz functions} If we set the learning rate function $\eta(t)=\frac{\ltwo{\C}}{\sqrt{t(n^2L^2+p\sigma^2)}}$, then we have the following excess risk bound. Here $L$ is the Lipscthiz constant of the loss function $\ell$. $$\E\left[\empL(\privtheta;\D)-\empL(\theta^*;\D)\right]=O\left(\frac{L\ltwo{\C}\log^{3/2}(n/\delta)\sqrt{p\log(1/\delta)}}{\epsilon}\right).$$

\item \mypar{Lipschitz and strongly convex functions} If we set the learning rate function $\eta(t)=\frac{1}{\Delta n t}$, then we have the following excess risk bound. Here $L$ is the Lipscthiz constant of the loss function $\ell$ and $\Delta$ is the strong convexity parameter. $$\E\left[\empL(\privtheta;\D)-\empL(\theta^*;\D)\right]=O\left(\frac{L^2\log^2(n/\delta)p\log(1/\delta)}{n\Delta\epsilon^2}\right).$$
\end{enumerate}
\label{thm:abcd4}
\end{thm}

\begin{proof}
Let $G_t=n\grad\ell(\htheta_t;d)+b_t$ in Line \ref{line:GD2} of Algorithm \ref{Algo:GradDesc}. First notice that over the randomness of the sampling of the data entry $d$ from $\D$, and the randomness of $b_t$, $\E\left[G_t\right]=\grad\empL(\htheta_t;\D)$. Additionally, we have the following bound on $\E[\ltwo{G_t}^2]$.
\begin{align}
\E[\ltwo{G_t}^2]&=n^2\E[\ltwo{\grad\ell(\htheta_t;d)}^2]+2n\E[\ip{\grad\ell(\htheta_t;d)}{b_t}]+\E[\ltwo{b_t}^2]\nonumber\\
&\leq n^2L^2+p\sigma^2\ \ \ \ \ \ \ \ \ \ \ \ \ \ \ \text{[Here $\sigma^2$ is the variance of $b_t$ in Line \ref{line:GD2}]}
\label{eq:GDabcd1}
\end{align}
In the above expression we have used the fact that since $\htheta_t$ is independent of $b_t$, so $\E[\ip{\grad\ell(\htheta_t;d)}{b_t}]=0$. Also, we have $\E[\ltwo{b_t}^2]=p\sigma^2$. We can now directly use Lemma \ref{thm:gradDes1SZ} to obtain the required error guarantee for Lipschitz convex functions, and Lemma \ref{thm:gradDes2SZ} for Lipschitz and strongly convex functions.

\begin{lem}[Theorem 2 from \citep{shamir2013stochastic}]
Let $F(\theta)$ (for $\theta\in\C$) be a convex function and let $\theta^*=\arg\min\limits_{\theta\in\C} F(\theta)$. Let $\theta_1$ be any arbitrary point from $\C$. Consider the stochastic gradient descent algorithm $\theta_{t+1}=\Pi_\C\left[\theta_t-\eta(t) G_t(\theta_t)\right]$, where $E[G_t(\theta_t)]=\grad F(\theta_t)$, $E[\ltwo{G_t}^2]\leq G^2$ and the learning rate function $\eta(t)=\frac{\ltwo{\C}}{G\sqrt t}$. Then for any $T>1$, the following is true.
$$\E\left[F(\theta_T)-F(\theta^*)\right]=O\left(\frac{\ltwo{\C}G\log T}{\sqrt T}\right).$$
\label{thm:gradDes1SZ}
\end{lem}
Using the bound from \eqref{eq:GDabcd1} in Lemma \ref{thm:gradDes1SZ} (i.e., set $G=\sqrt{n^2L^2+p\sigma^2}$), and setting $T=n^2$ and the learning rate function $\eta(t)$ as in Lemma \ref{thm:gradDes1SZ}, gives us the required excess risk bound for Lipschitz convex functions. For Lipschitz and strongly convex functions we use the following result by \cite{shamir2013stochastic}.

\begin{lem}[Theorem 1 from \citep{shamir2013stochastic}]
Let $F(\theta)$ (for $\theta\in\C$) be a $\lambda$-strongly convex function and let $\theta^*=\arg\min\limits_{\theta\in\C} F(\theta)$. Let $\theta_1$ be any arbitrary point from $\C$. Consider the stochastic gradient descent algorithm $\theta_{t+1}=\Pi_\C\left[\theta_t-\eta(t) G_t(\theta_t)\right]$, where $E[G_t(\theta_t)]=\grad F(\theta_t)$, $E[\ltwo{G_t}^2]\leq G^2$ and the learning rate function $\eta(t)=\frac{1}{\lambda t}$. Then for any $T>1$, the following is true.
$$\E\left[F(\theta_T)-F(\theta^*)\right]=O\left(\frac{G^2\log T}{\lambda T}\right).$$
\label{thm:gradDes2SZ}
\end{lem}
Using the bound from \eqref{eq:GDabcd1} in Lemma \ref{thm:gradDes2SZ} (i.e., set $G=\sqrt{n^2L^2+p\sigma^2}$), $\lambda=n\Delta$, and setting $T=n^2$ and the learning rate function $\eta(t)$ as in Lemma \ref{thm:gradDes2SZ}, gives us the required excess risk bound for Lipschitz and strongly convex convex functions.
\end{proof}

\mypar{Note} Algorithm $\A_{\sf Noise-GD}$ has a running time of $O(pn^2)$, assuming that the gradient computation for $\ell$ takes time $O(p)$. Variants of Algorithm $\A_{\sf Noise-GD}$ have appeared in \cite{WM10,JKT12,DuchiJW13,song2013stochastic}. The most relevant work in our current context is that of \cite{song2013stochastic}. The main idea in \cite{song2013stochastic} is to run stochastic gradient descent with gradients computed over small batches of \emph{disjoint} samples from the data set (as opposed to one single sample used in Algorithm $\A_{\sf Noise-GD}$). The issue with the algorithm is that it cannot provide excess risk guarantee which is $o(\sqrt n)$, where $n$ is the number of data samples. One observation that we make is that if one removes the constraint of disjointness and use the amplification lemma (Lemma \ref{claim:privAmp}), then one can ensure a much tighter privacy guarantees for the same setting of parameters used in the paper. 


\section{Exponential Sampling and Optimal $(\epsilon,0)$-private Optimization}
\label{sec:privConvexUpper}

In this section, we focus on the case of pure $\epsilon$-differential privacy and provide an optimal efficient algorithm for empirical risk minimization for the general class of convex and Lipschitz loss functions. The main building block of this section is the well-known exponential mechanism \cite{MT07}.

First, we show that a variant of the exponential mechanism is optimal. A major technical contribution of this section is to make the exponential mechanism computationally efficient which is discussed in Section~\ref{sec:eff-pure-eps-alg}.

\subsection{Exponential Mechanism for Lipschitz Convex Loss}
\label{sec:lipschitzConvex}

In this section we only deal with loss functions which are Lipschitz. We provide an $\epsilon$-differentially private algorithm (Algorithm \ref{Alg:GenSamp}) which achieves the optimal excess risk for arbitrary convex bounded sets.

\begin{algorithm}[htb]
	\caption{$\A_{\sf exp-samp}$: Exponential sampling based convex optimization}
	\begin{algorithmic}[1]
		\REQUIRE Data set of size $n$: $\D$, loss function $\ell$, privacy parameter $\epsilon$ and convex set $\C$.
        \STATE $\empL(\theta;\D)=\sum\limits_{i=1}^n\ell(\theta;d_i)$.
        \STATE Sample a point $\privtheta$ from the convex set $\C$ w.p. proportional to $\exp\left(-\frac{\epsilon}{2L\ltwo{\C}}\empL(\theta;\D)\right)$ and output. \label{step2}
	\end{algorithmic}
	\label{Alg:GenSamp}
\end{algorithm}

\begin{thm}[Privacy guarantee]
Algorithm \ref{Alg:GenSamp} is $\epsilon$-differentially private.
\label{thm:privVol}
\end{thm}

\begin{proof} First, notice that the distribution induced by the exponential weight function in step~\ref{step2} of Algorithm~\ref{Alg:GenSamp} is the same if we use $\exp\left(-\frac{\epsilon}{L\ltwo{\C}}\left(\empL(\theta;\D)-\empL(\theta_0;\D)\right)\right)$ for some arbitrary point $\theta_0\in\C$. Since $\ell$ is $L$-Lipschitz, the sensitivity of $\empL(\theta;\D)-\empL(\theta_0;\D)$ is at most $L\ltwo{\C}$. The proof then follows directly from the analysis of \emph{exponential mechanism} by \cite{MT07}.
\end{proof}

In the following we prove the utility guarantee for Algorithm $\A_{\sf exp-samp}$.

\begin{thm}[Utility guarantee]
 Let $\privtheta$ be the output of $\A_{\sf exp-samp}$ (Algorithm \ref{Alg:GenSamp} above). Then, we have the following guarantee on the expected excess risk. (The expectation is over the randomness of the algorithm.)
$$\E\left[\empL(\privtheta;\D)-\empL(\theta^*;\D)\right]=O\left(\frac{pL\ltwo{\C}}{\epsilon}\right).$$
\label{thm:abcd2}
\end{thm}

\begin{proof}
Consider a differential cone $\Omega$ centered at $\theta^*$ (see Figure \ref{fig:circles1}). We will bound the expected excess risk of $\privtheta$ by $O\left(\frac{pL\ltwo{\C}}{\epsilon}\right)$ conditioned on $\privtheta\in\Omega\cap\C$ for every differential cone. This immediately implies the above theorem by the properties of conditional expectation.

\begin{figure}
\begin{center}
    \includegraphics[scale=0.45]{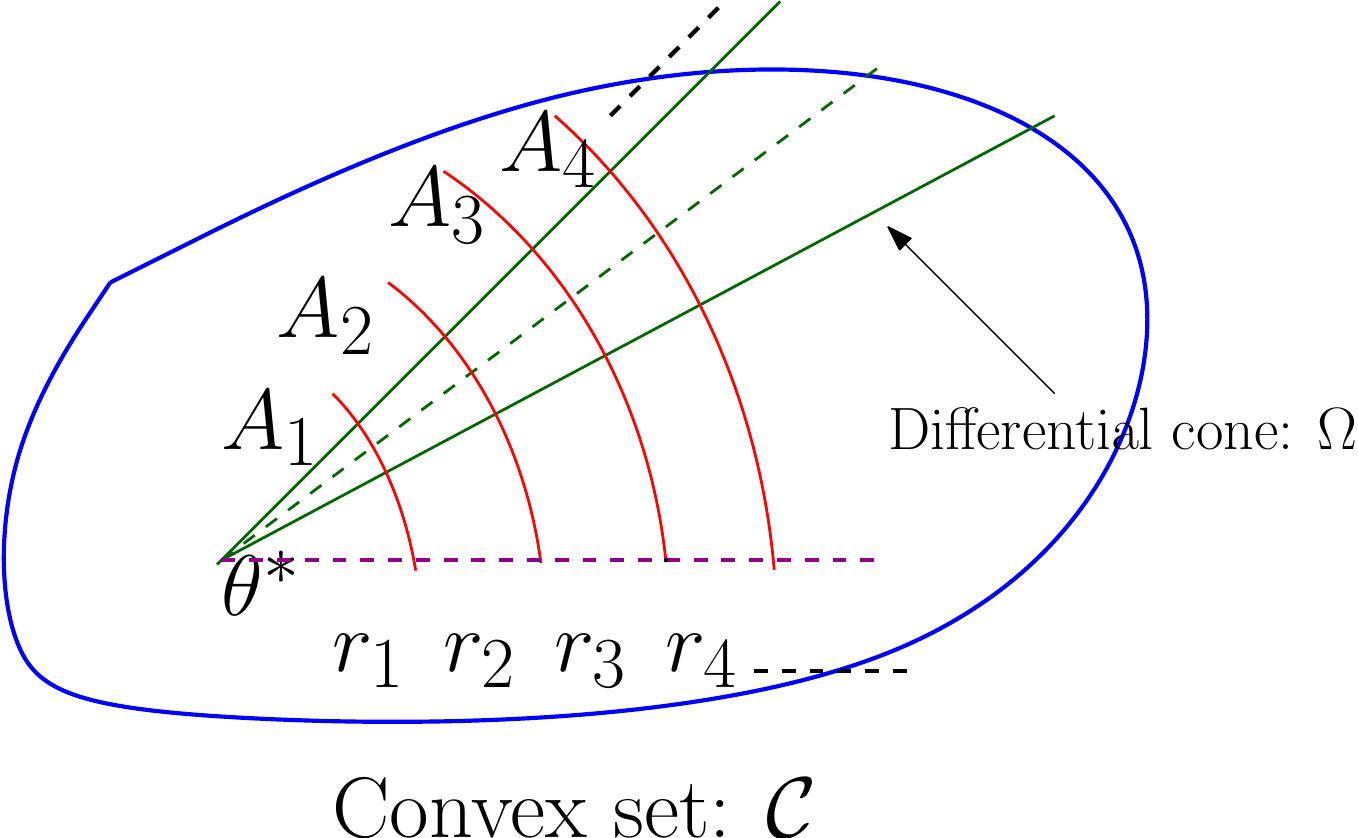}
\end{center}
\caption{Differential cone $\Omega$ inside the convex set $\C$}
\label{fig:circles1}
\end{figure}

Let $\Gamma$ be a fixed threshold (to be set later) and let $\risk(\theta)=\empL(\privtheta;\D)-\empL(\theta^*;\D)$ for the purposes of brevity. Let the marked sets $A_i$'s in Figure \ref{fig:circles1} be defined as
$${\small A_i=\{\theta\in\Omega\cap\C: (i-1)\Gamma\leq\risk(\theta)\leq i\cdot \Gamma\}.}$$
Instead of directly computing the probability of $\privtheta$ being outside $A_1$, we will analyze the probabilities for being in each of the $A_i$'s individually. This form of ``peeling'' arguments have been used for risk analysis of convex loss in the machine learning literature (e.g., see \cite{sridharan2008fast}) and will allow us to get rid of the extra logarithmic factor that would have otherwise shown up in the excess risk if we use the standard analysis of the exponential mechanism in \cite{MT07}.

Since $\Omega$ is a differential cone and since $\risk(\theta)$ is continuous on $\C$, it follows that within $\Omega\cap\C$, $\risk(\theta)$ only depends on $\ltwo{\theta-\theta^*}$. Therefore, let $r_1,r_2,\cdots$ be the distance of the set boundaries of $A_1,A_2,\cdots$ from $\theta^*$. (See Figure \ref{fig:circles1}.) One can equivalently write each $A_i$ as follows:
$${\small A_i=\{\theta\in\Omega\cap\C: r_{i-1}<\ltwo{\theta-\theta^*}\leq r_i\}.}$$
The following claim is the key part of the proof. 
\begin{claim}
Convexity of $\risk(\theta)$ for all $\theta\in\C$ implies that $r_i-r_{i-1}\leq r_{i-1}-r_{i-2}$ for all $i\geq 3$.
\label{cl:obs12}
\end{claim}
\begin{proof}
Since by definition $\theta^*$ is the minimizer of $\risk(\theta)$ within $\C$ and $\risk(\theta)$ is convex, we have $\risk(\theta_2)\geq \risk(\theta_1)$ for any $\theta_1,\theta_2\in\C\cap\Omega$ such that $\ltwo{\theta_2-\theta^*}\geq\ltwo{\theta_1-\theta^*}$. This directly implies the required bound.
\end{proof}
Now, the volume of the set $A_i$ is given by ${\vol(A_i)}=\kappa\int\limits_{r_{i-1}}^{r_i} r^{p-1} dr$ for some fixed constant $\kappa$. Hence, 
$${\small\frac{\vol(A_i)}{\vol(A_2)}=\frac{r_{i-1}^p}{r_1^p}\cdot\frac{(r_i/r_{i-1})^p-1}{(r_2/r_1)^p-1}\leq\frac{r_{i-1}^p}{r_1^p}\leq (i-1)^p.}$$
where the last two inequalities follows from Claim~\ref{cl:obs12}. 
Hence, we get the following bound on the probability that the excess risk $\risk(\privtheta)\geq 4\Gamma$ conditioned on $\privtheta\in\C\cap\Omega$ (For brevity, we remove the conditioning sign from the probabilities below).
\begin{align}
\Pr[\risk(\privtheta)\geq 4\Gamma]&\leq\frac{\Pr[\privtheta\in \bigcup\limits_{i=4}^\infty A_i]}{\Pr[\privtheta\in A_2]}\leq\sum\limits_{i=4}^{\infty}\frac{\vol(A_i)}{\vol(A_2)}\cdot e^{-\epsilon(i-3)\frac{\Gamma}{2L\ltwo{\C}}}\leq\sum\limits_{i=4}^{\infty}(i-1)^p \cdot e^{-\epsilon(i-3)\frac{\Gamma}{2L\ltwo{\C}}}\nonumber\\
&\leq\frac{3^p e^{-\epsilon\frac{\Gamma}{2L\ltwo{\C}}}}{1-2^{p}e^{-\epsilon\frac{\Gamma}{2L\ltwo{\C}}}}\nonumber
\end{align}
where the last inequality follows from the fact that $(i-1)^p\leq 3^{p}\cdot\left(2^{i-4}\right)^p$ for $i\geq 4$. 
Hence, for every $t>0$, if we choose $\Gamma=\frac{2L\ltwo{\C}}{\epsilon}\left(\left(p+1\right)\ln 3+t\right)$, then, conditioned on $\privtheta\in\C\cap\Omega$, we get
$$\Pr[\risk(\privtheta)\geq \frac{8L\ltwo{\C}}{\epsilon}\left(\left(p+1\right)\ln 3+t\right)]\leq e^{-t}.$$ 
Since this is true for every $t>0$, we have our required bound as a corollary.

\end{proof}


\subsection{Efficient Implementation of Algorithm $\A_{\sf exp-samp}$ (Algorithm \ref{Alg:GenSamp})}
\label{sec:eff-pure-eps-alg}

In this section, we give a high-level description of a computationally efficient construction of Algorithm \ref{Alg:GenSamp}. Our algorithm runs in polynomial time in $n, p$ and outputs a sample $\theta\in\C$ from a distribution that is arbitrarily close (in the multiplicative sense) to the distribution of the output of Algorithm~\ref{Alg:GenSamp}.

Since we are interested in an efficient pure $\epsilon$-differentially private algorithm, we need an efficient sampler with a multiplicative distance guarantee. In fact, if we were interested in $(\eps, \delta)$ algorithms, efficient sampling with a total variation guarantee would have sufficed which would have made our task a lot easier as we could have used one of the exisiting algorithms, e.g., \cite{LovaszV07}. In \cite{HT09}, it was shown how to sample efficiently with a multiplicative guarantee from the \emph{unifrom} distribution over a convex bounded set.  However, what we want to achieve here is more general, that is, to sample efficiently from any given logconcave distribution defined over a convex bounded set. To the best of our knowledge, this task has not been explicitly worked out before, nevertheless, all the ingredients needed to accomplish it are present in the literature, mainly \cite{applegate1991sampling}.


We highlight here the main ideas of our constrution. Since such construction is not specific to our privacy problem and could be of independent interest, in this section, we only provide the high-level description of this construction, however all the details of such construction and the proof of our main result  (Theorem~\ref{thm:eff-samp-guarantees} below) are deferred to Section~\ref{sec:implement-eff-exp}.
%

\begin{thm}
There is an efficient version of Algorithm \ref{Alg:GenSamp} that has the following guarantees.
\begin{enumerate}
\item \mypar{Privacy} The algorithm is $\epsilon$ -differentially private.\vspace{-0.2cm}
\item \mypar{Utility} The output $\privtheta\in\C$ of the algorithm satisfies
$$\E\left[\empL(\privtheta;\D)-\empL(\theta^*;\D)\right]=O\left(\frac{pL\ltwo{\C}}{\epsilon}\right).$$\vspace{-0.4cm}
\item \mypar{Running time} Assuming $\C$ is in isotropic position, the algorithm runs in time\footnote{The case where $\C$ is not in isotropic position is discussed below.}
$$O\left(\ltwo{\C}^2p^3n^3\max\left\{p\log(\ltwo{\C}pn), \epsilon \ltwo{\C}  n \right\}\right).$$
\end{enumerate}
\label{thm:eff-samp-guarantees}
\end{thm}
In fact, the running time of our algorithm depends on $\linf{\C}$ rather than $\ltwo{\C}$. Namely, all the $\ltwo{\C}$ terms in the running time can be replaced with $\linf{\C}$, however, we chose to write it in this less conservative way since all the bounds in this paper are expressed in terms of $\ltwo{\C}$. 

Before describing our construction, we first introduce some useful notation and discuss some preliminaries. 

For any two probability measures $\mu, \nu$ defined with respect to the same sample space $\Q\subseteq\re^p$, the relative (multiplicative) distance between $\mu$ and $\nu$, denoted as $\dist(\mu,\nu)$ is defined as
$$\dist(\mu, \nu)=\sup\limits_{q\in\Q}\left\vert\log\frac{d\mu(q)}{d\nu(q)}\right\vert.$$
where $\frac{d\mu(q)}{d\nu(q)}$ (resp., $\frac{d\nu(q)}{d\mu(q)}$) denotes the ratio of the two measures (more precisely, the Radon-Nikodym derivative).

\mypar{Assumptions} We assume that we can efficiently test whether a given point $\theta\in\re^{p}$ lies in $\C$ using a membership oracle. We also assume that we can efficienly optimize an efficiently computable convex function over a convex set. To do this, it suffices to have a projection oracle. We do not take into account the extra polynomial factor in the running time which is required to perform such operations since this factor is highly dependent on the specific structure of the set $\C$.


\mypar{When $\C$ is not isotropic} In Theorem~\ref{thm:eff-samp-guarantees} and in the rest of this section, we assume that $\C$ is in isotropic position. In particular, we assume that $\B\subseteq\C\subseteq p\B$. However, if the convex set is not in isotropic position, fortunately, we know of efficient algorithms for placing it in isotropic position (for example, the algorithm of \cite{Lov-Vemp03}). In such case, the first step of our algorithm would be to transform $\C$ to an isotropic position (and apply the corresponding transformation to the loss function). This step takes time polynomial in $p$ with additional polylog factor in $\frac{1}{r}$ where $r>0$ is the diameter of the largest ball we can fit inside $\C$ (See \cite{Lov-Vemp03} and \cite{flaxman2005online}). Specifically, if $r=\Omega(1)$, then our set $\C$ would be already in isotropic position. Then, we sample efficiently from the transformed set. Finally, we apply the inverse transformation to the output sample to obtain a sample from the desired distribution over $\C$ in its original position. In the worst case, the isotropic transformation can amplify the diameter of $\C$ by a factor of $p$. Putting all this together, the running time in Theorem~\ref{thm:eff-samp-guarantees} above will pick up an extra factor of $O\left(\max\left(p^3,\mbox{polylog}\left(\frac{1}{r}\right)\right)\right)$ in the worst case if $\C$ is not in isotropic position.

\subsection{Our construction}

Let $\tau$ denote the $L_{\infty}$ diameter of $\C$. The Minkowski's norm of $\theta\in\re^p$ with respect to $\C$, denoted as $\psi(\theta)$, is defined as $\psi(\theta)=\inf\{r>0: \theta\in r\C\}$. We define $\bar{\psi}_{\alpha}(\theta)\triangleq\alpha\cdot\max\{0, \psi(\theta)-1\}$ for $\alpha>0$. Note that $\bar{\psi}_{\alpha}(\theta)>0$ if and only if $\theta\notin\C$. Moreover, it is not hard to verify that $\bar{\psi}_{\alpha}$ is $\alpha$-Lipschitz.

We use the grid-walk algorithm of \cite{applegate1991sampling} for sampling from a logconcave distribution defined over a \emph{cube} as a building block. Our construction is described as follows: 

\begin{enumerate}

\item Enclose the set $\C$ with a cube $A$ with edges of length $\tau$. 

\item Obtain a convex Lipschitz extension $\bar{\empL}(. ; \D)$ of the loss function $\empL(. ;\D)$ over $A$. This can be done efficiently using a projection oracle. 

\item Define $F(\theta)\triangleq e^{-\frac{\eps}{6L\ltwo{\C}}\bar{\empL}(\theta; \D)-\bar{\psi}_{\alpha}(\theta)},~\theta\in A$, for a specific choice of $\alpha=O(\frac{\eps n}{\ltwo{C}})$ (See Section~\ref{sec:implement-eff-exp} for details). 

\item Run the grid-walk algorithm of \cite{applegate1991sampling} with $F$ as the input weight function and $A$ as the input cube, and output a sample $\theta$ whose distribution is close, with respect to $\dist$, to the distribution induced by $F$ on $A$ which is given by $\frac{F(\theta)}{\int\limits_{v\in A}F(v)dv},~\theta\in A$. 
\end{enumerate}

Now, note that what we have so far is an efficient procedure (let's denote it by $\A_{\sf cube-samp}$) that outputs a sample from a distribution over $A$  which close, with respect to $\dist$, to the continuous distribution $\frac{F(u)}{\int\limits_{v\in A}F(v)dv},~u\in A$. We then argue that due to the choices made for the values of the parameters above, $\A_{\sf cube-samp}$ outputs a sample in $\C$ with probability at least $\frac{1}{2}$. That is, the algorithm succeeds to output a sample from a distribution close to the right distribution on $\C$ with probability at least $1/2$. Hence, we can amplify the probability of success by repeating $\A_{\sf cube-samp}$ sufficiently many times where fresh random coins are used by $\A_{\sf cube-samp}$ in every time (specifically, $O(n)$ iterations would suffice). If $\A_{\sf cube-samp}$ returns a sample $\theta\in\C$ in one of those iterations, then our algorithm terminates outputting $\theta$. Otherwise, it outputs a uniformly random sample $\theta^\perp$ from the unit ball $\B$ (Note that $\B\subseteq\C$ since $\C$ is in isotropic position). We finally show that this termination condition can only change the distribution of the output sample by a constant factor sufficiently close to $1$. Hence, we obtain our efficient algorithm referred to in Theorem~\ref{thm:eff-samp-guarantees}.

\section{Localization and Optimal Private Algorithms for Strongly Convex Loss}\label{sec:out-pert}

It is unclear how to get a direct variant of Algorithm \ref{Alg:GenSamp} in Section \ref{sec:privConvexUpper} for Lipschitz and strongly convex losses that can achieve optimal excess risk guarantees. The issue in extending Algorithm \ref{Alg:GenSamp} directly is that the convex set $\C$ over which the exponential mechanism is defined is ``too large'' to provide tight guarantees. 

We show a generic $\epsilon$-differentially private algorithm  for minimizing Lipschitz strongly convex loss functions based on a combination of a simple pre-processing step (called the \emph{localization step}) and any generic $\epsilon$-differentially private algorithm for Lipschitz convex loss functions. We carry out the localization step using a simple output perturbation algorithm which ensures that the convex set over which the $\epsilon$-differentially private algorithm (in the second step) is run has diameter $\tilde O(p/n)$. 

Next, we instantiate the generic $\epsilon$-differentially private algorithm in the second step with our efficient exponential mechanism of Section\ref{sec:lipschitzConvex} (Algorithm \ref{Alg:GenSamp}) to obtain an algorithm with optimal excess risk bound (Theorem \ref{thm:utility_eps_st_cnvx_vol_samp}).

\mypar{Note} The localization technique is not specific to pure $\epsilon$-differential privacy, and extends naturally to $(\epsilon,\delta)$ case. Although it is not relevant in our current context, since we already have gradient descent based algorithm which achieves optimal excess risk bound. We defer the details for the $(\epsilon,\delta)$ case to Appendix \ref{sec:localization-eps-delta}.



\mypar{Details of the generic algorithm} We first give a simple algorithm that carries out the desired localization step. The crux of the algorithm is the same as to that of the output perturbation algorithm of \citep{CM08,CMS11}. The high-level idea is to first compute $\theta^*=\arg\min\limits_{\theta\in\C}\empL(\theta;\D)$ and add noise according to the sensitivity of $\theta^*$. The details of the algorithm are given in Algorithm \ref{Alg:OutPert}.

\begin{algorithm}[htb]
	\caption{$\A^{\epsilon}_{\sf out-pert}$: Output Perturbation for Strongly Convex Loss}
	\begin{algorithmic}[1]
		\REQUIRE data set of size $n$: $\D$, loss function $\ell$, strong convexity parameter $\Delta$, privacy parameter $\epsilon$, convex set $\C$, and radius parameter $\zeta<1$.
        \STATE $\empL(\theta;\D)=\sum\limits_{i=1}^n\ell(\theta;d_i)$.
	\STATE Find $\theta^*=\arg\min\limits_{\theta\in\C}\empL(\theta; \D)$.
        \STATE $\theta_0=\Pi_\C\left(\theta^*+b\right)$, where $b$ is random noise vector with density $\frac{1}{\alpha}e^{-\frac{n\Delta\epsilon}{2L}\ltwo{b}}$ (where $\alpha$ is a normalizing constant) and $\Pi_\C$ is the projection on to the convex set $\C$.
        \STATE Output $\C_0=\{\theta\in\C: \ltwo{\theta-\theta_0}\leq \zeta\frac{2Lp}{\Delta\epsilon n}\}$.
	\end{algorithmic}
	\label{Alg:OutPert}
\end{algorithm}

Having Algorithm~\ref{Alg:OutPert} in hand, we now give a generic $\epsilon$-differentially private algorithm for minimizing $\empL$ over $\C$. Let $\A^{\epsilon}_{\sf gen-Lip}$ denote any generic $\epsilon$-differentially private algorithm for optimizing $\empL$ over some arbitrary convex set $\tilde{\C}\subseteq\C$.  Algorithm~\ref{Alg:GenSamp} from Section~\ref{sec:lipschitzConvex} or its efficient version Algorithm~\ref{Alg:eff-exp-samp}(See Theorem~\ref{thm:eff-samp-guarantees} and Section~\ref{sec:implement-eff-exp} for details) is an example of $\A^{\epsilon}_{\sf gen-Lip}$. The algorithm we present here (Algorithm~\ref{Alg:eps-gen-st-cnvx} below) makes a black-box call in its first step to $\A^{\frac{\epsilon}{2}}_{\sf out-pert}$ (Algorithm~\ref{Alg:OutPert} shown above), then, in the second step, it feeds the output of $\A^{\frac{\epsilon}{2}}_{\sf out-pert}$ into $\A^{\frac{\epsilon}{2}}_{\sf gen-Lip}$ and ouptut.

\begin{algorithm}[htb]
	\caption{Output-perturbation-based Generic Algorithm}
	\begin{algorithmic}[1]
		\REQUIRE data set of size $n$: $\D$, loss function $\ell$, strong convexity parameter $\Delta$, privacy parameter $\epsilon$, and convex set $\C$.
        \STATE Run $\A^{\frac{\epsilon}{2}}_{\sf out-pert}$ (Algorithm~\ref{Alg:OutPert}) with input privacy parameter $\epsilon/2$, radius parameter $\zeta=3\log{(n)}$ and output $\C_0$.
        \STATE Run $\A^{\frac{\epsilon}{2}}_{\sf gen-Lip}$ on inputs $n, \D, \ell,$ privacy parameter $\epsilon/2$, and convex set $\C_0$, and output $\privtheta$.
	\end{algorithmic}
	\label{Alg:eps-gen-st-cnvx}
\end{algorithm}

\begin{thm}[Privacy guarantee]
Algorithm \ref{Alg:eps-gen-st-cnvx} is $\epsilon$-differentially private.
\label{thm:priv_eps_Gen_st_cnvx}
\end{thm}

\begin{proof}
The privacy guarantee follows directly from the composition theorem together with the fact that $\A^{\frac{\epsilon}{2}}_{\sf out-pert}$ is $\frac{\epsilon}{2}$-differentially private (see \cite{CMS11}) and that $\A^{\frac{\epsilon}{2}}_{\sf gen-Lip}$ is $\frac{\epsilon}{2}$-differentially private by assumption.
 \end{proof}

In the following theorem, we provide a generic expression for the excess risk of Algorithm~\ref{Alg:eps-gen-st-cnvx} in terms of the expected excess risk of any given algorithm $\A_{\sf gen-Lip}$.

\begin{lem}[Generic utility guarantee]
Let $\htheta$ denote the output of Algorithm $\A^{\epsilon}_{\sf gen-Lip}$ on inputs $n, \D, \ell, \epsilon, \tilde{\C}$ (for an arbitrary convex set $\tilde{\C}\subseteq\C$). Let $\nptheta$ denote the minimizer of $\empL(. ; \D)$ over $\tilde{\C}$. If
$$\E\left[\empL(\htheta; \D)-\empL(\nptheta; \D)\right]\leq F\left(p, n, \epsilon, L,\ltwo{\tilde{\C}}\right)$$
for some function $F$, then the output $\privtheta$ of Algorithm~\ref{Alg:eps-gen-st-cnvx} satisfies
$$\E\left[\empL(\privtheta; \D)-\empL(\theta^*; \D)\right]= O\left(F\left(p, n, \epsilon, L, O\left(\frac{Lp\log(n)}{\Delta\epsilon n}\right)\right)\right),$$
where $\theta^*=\arg\min\limits_{\theta\in\C}\eL(\theta; \D)$.
\label{thm:utility_eps_Gen_st_cnvx}
\end{lem}
\begin{proof}
The proof follows from the fact that, in Algorithm $\A^{\frac{\epsilon}{2}}_{\sf out-pert}$, the norm of the noise vector $\ltwo{b}$ is distributed according to Gamma distribution $\Gamma(p, \frac{4L}{\Delta\epsilon n})$ and hence satisfies
$$\Pr\left(\ltwo{b}\leq \zeta\frac{4Lp}{\Delta\epsilon n}\right)\geq 1- e^{-\zeta}$$
(see, for example, \cite{CMS11}). Now, set $\zeta=3\log(n)$. Hence, with probability $1-\frac{1}{n^3}$,
$\C_0$ (the output of $\A^{\frac{\epsilon}{2}}_{\sf out-pert}$) contains $\theta^*$. Hence, by setting $\tilde{\C}$ in the statement of the lemma to $\C_0$ (and noting that $\ltwo{\C_0}=O\left(\frac{Lp\log(n)}{\Delta\epsilon n}\right)$), then \emph{conditioned on} the event that $\C_0$ contains $\theta^*$, we have $\nptheta=\theta^*$ and hence
$$E\left[\empL(\privtheta; \D)-\empL(\theta^*; \D)~\vert \theta^*\in\C_0\right]\leq F\left(p, n, \epsilon, L, O\left(\frac{Lp\log(n)}{\Delta\epsilon n}\right)\right)$$
Thus,
$$E\left[\empL(\privtheta; \D)-\empL(\theta^*; \D)\right]\leq F\left(p, n, \epsilon, L, O\left(\frac{Lp\log(n)}{\Delta\epsilon n}\right)\right)(1-\frac{1}{n^3}) + nL\ltwo{\C}\frac{1}{n^3}$$
Note that the second term on the right-hand side above becomes $O(\frac{1}{n^2})$. From our lower bound (Section~\ref{sec:lipschitzConvexLowerStrong} below), $F(. , n, . , . , .)$ must be at least $\Omega(\frac{1}{n})$. Hence, we have
$$E\left[\empL(\privtheta; \D)-\empL(\theta^*; \D)\right]= O\left(F\left(p, n, \epsilon, L, O\left(\frac{Lp\log(n)}{\Delta\epsilon n}\right)\right)\right)$$
which completes the proof of the theorem.
\end{proof}

\mypar{Instantiation of Algorithm $\A^{\frac{\epsilon}{2}}_{\sf gen-Lip}$ with the exponential sampling algorithm}  Next, we give our optimal $\epsilon$-differentially private algorithm for Lipschitz strongly convex loss functions. To do this, we instantiate the generic Algorithm $\A_{\sf gen-Lip}$ in Algorithm \ref{Alg:eps-gen-st-cnvx} with our exponential sampling algorithm from Section~\ref{sec:lipschitzConvex} (Algorithm \ref{Alg:GenSamp}), or its efficient version Algorithm~$\A_{\sf eff-exp-samp}$ (See Section~\ref{sec:eff-pure-eps-alg}) to obtain the optimal excess risk bound. We formally state the bound in Theorem \ref{thm:utility_eps_st_cnvx_vol_samp}) below. The proof of Theorem \ref{thm:utility_eps_st_cnvx_vol_samp} follows from Theorem \ref{thm:abcd2} and Lemma~\ref{thm:utility_eps_Gen_st_cnvx} above.

\begin{thm}[Utility guarantee with Algorithm~\ref{Alg:eff-exp-samp} as an instantiation of $\A_{\sf gen-Lip}$]
Suppose we replace  $\A^{\frac{\epsilon}{2}}_{\sf gen-Lip}$ in Algorithm~\ref{Alg:eps-gen-st-cnvx} with Algorithm~\ref{Alg:GenSamp} (Section~\ref{sec:lipschitzConvex}), or its efficient version Algorithm~\ref{Alg:eff-exp-samp} (See Theorem~\ref{thm:eff-samp-guarantees} and Section~\ref{sec:implement-eff-exp} for details). Then, the output $\privtheta$ satisfies
$$E\left[\empL(\privtheta; \D)-\empL(\theta^*; \D)\right]= O\left(\frac{p^2L^2}{n\Delta\epsilon^2}\log(n)\right)$$
where $\theta^*=\arg\min\limits_{\theta\in\C}\eL(\theta; \D)$.
\label{thm:utility_eps_st_cnvx_vol_samp}
\end{thm}





\section{Lower Bounds on Excess Risk}
\label{sec:privConvexLower}

In this section, we complete the picture by deriving lower bounds on the excess risk caused by differentially private algorithm for risk minimization.  As before, for a dataset $\D=\{d_1, \ldots, d_n\}$, our decomposable loss  function is expressed as $\eL(\theta; \D)=\sum\limits_{i=1}^n\ell(\theta; d_i),~\theta\in\C$ for some convex set $\C\subset\re^p$. In Section~\ref{sec:lipschitzConvexLower}, we consider the case of convex Lipschitz loss functions, whereas in Section~\ref{sec:lipschitzConvexLowerStrong}, we consider the case of strongly convex and Lipschitz loss functions.

Before we state and prove our lower bounds, we first give the following useful lemma which gives lower bounds on the $L_2$-error incurred by $\eps$ and $(\eps, \delta)$-differentially private algorithms for estimating the $1$-way marginals of datasets over $\hypcnz^p$. This lemma is based on the results of \cite{HT09} and \cite{BUV13}, however, for the sake of completeness, we give a detailed proof of this lemma in Appendix~\ref{app:proof-lower-bounds-lem}. 

\begin{lem}[Lower bounds for 1-way marginals]
\hspace{7cm} 
\begin{enumerate}
\item \mypar{$\boldsymbol\epsilon$-differential private algorithms} Let $n, p\in\mathbb{N}$ and $\epsilon > 0$. There is a number $M=\Omega\left(\min\left( n, p/\epsilon\right)\right)$ such that for every $\eps$-differentially private algorithm $\A$, there is a dataset $\D=\{d_1, \ldots, d_n\} \subseteq \hypcnz^p$ with $\ltwo{\sum_{i=1}^nd_i}\in [M-1, M+1]$ such that, with probability at least $1/2$ (taken over the algorithm random coins), we have 
$$ \ltwo{\A(\D)-q(\D)}=\Omega\left(\min\left(1, \frac{p}{\epsilon n}\right)\right)$$
where $q(\D)=\frac{1}{n}\sum_{i=1}^{n}d_i$. 

\item \mypar{$(\boldsymbol\epsilon,\boldsymbol\delta)$-differential private algorithms}  Let $n, p\in\mathbb{N}$, $\epsilon>0$, and $\delta=o(\frac{1}{n})$. There is a number $M=\Omega\left(\min\left( n, \sqrt{p}/\epsilon\right)\right)$ such that for every $(\eps,\delta)$-differentially private algorithm $\A$, there is a dataset $\D=\{d_1, \ldots, d_n\} \subseteq \hypcnz^p$ with ${\small\ltwo{\sum_{i=1}^nd_i}\in[M-1, M+1]}$ such that, with probability at least $1/3$ (taken over the algorithm random coins), we have 
$$l \ltwo{\A(\D)-q(\D)}=\Omega\left(\min\left(1, \frac{\sqrt{p}}{\epsilon n}\right)\right)$$
where $q(\D)=\frac{1}{n}\sum_{i=1}^{n}d_i$.
\end{enumerate}
\label{lem:1-way-marg-lower}
\end{lem}

%

\subsection{Lower bounds for Lipschitz Convex Functions}
\label{sec:lipschitzConvexLower} 

In this section, we give lower bounds for both $\epsilon$ and $(\epsilon, \delta)$ differentially private algorithms for minimizing any convex Lipschitz loss function $\empL(\theta; \D)$. We consider the following loss function. Define 
\begin{equation}
\ell(\theta; d)= - \langle \theta, d\rangle,~\theta\in\B, ~d\in\hypcnz^p\label{lin-loss}
\end{equation}
For any dataset $\D=\{d_1, \ldots, d_n\}$ with data points drawn from $\hypcnz^p$, and any $\theta\in\B$, define 
\begin{equation}
\empL(\theta; \D)=-\langle\theta, \sum_{i=1}^n d_i\rangle \label{lin-dec-loss}
\end{equation}
Clearly, $\empL$ is linear and, hence, Lipschitz and convex. Note that, whenever $\ltwo{\sum_{i=1}^nd_i}>0$, $\theta^*=\frac{\sum_{i=1}^nd_i}{\ltwo{\sum_{i=1}^nd_i}}$ is the minimizer of $\eL( . ;\D)$ over $\B$. Next, we show lower bounds on the excess risk incurred by any $\epsilon$ and $(\epsilon, \delta)$ differentially private algorithm with output $\privtheta\in\B$.

\begin{thm}[Lower bound for $\epsilon$-differentially private algorithms]
Let $n, p\in\mathbb{N}$ and $\epsilon>0$. For every $\epsilon$-differentially private algorithm (whose output is denoted by $\privtheta$), there is a dataset $\D=\{d_1, \ldots, d_n\} \subseteq \hypcnz^p$ such that, with probability at least $1/2$ (over the algorithm random coins), we must have 
$$\empL(\theta; \D)-\empL(\theta^*; \D)= \Omega\left(\min\left(n, p/\epsilon\right)\right)$$ 
where $\theta^*=\frac{\sum_{i=1}^nd_i}{\ltwo{\sum_{i=1}^nd_i}}$ is the minimizer of $\empL( . ; \D)$ over $\B$ and $\empL$ is given by (\ref{lin-dec-loss}). 
\label{thm:lower-lip-eps}
\end{thm}

\begin{proof}
Let $\A$ be an $\eps$-differentially private algorithm for minimizing $\eL$ and let $\privtheta$ denote its output. First, observe that for any $\theta\in\B$ and dataset $\D$, $\empL(\theta; \D)-\empL(\theta^*; \D)=\ltwo{\sum_{i=1}^nd_i}\left(1-\langle\theta, \theta^*\rangle\right)$. Hence, we have $\empL(\theta; \D)-\empL(\theta^*; \D)\geq\frac{1}{2}\ltwo{\sum_{i=1}^nd_i}\ltwo{\theta-\theta^*}^2$. This is due to the fact that $\ltwo{\theta-\theta^*}^2=\ltwo{\theta^*}^2 + \ltwo{\theta}^2- 2\langle \theta, \theta^*\rangle$ and the fact that $\theta^*, \theta \in\B$.

Let $M=\Omega\left(\min\left( n, p/\epsilon\right)\right)$ be as in Part~1 of Lemma~\ref{lem:1-way-marg-lower}. Suppose, for the sake of a contradiction, that for every dataset $\D\subseteq\hypcnz^p$ with $\ltwo{\sum_{i=1}^n d_i}\in [M-1, M+1]$, with probability more than $1/2$, we have $\ltwo{\privtheta-\theta^*}\neq \Omega\left(1\right)$. Let $\tilde\A$ be an $\eps$-differentially private algorithm that first runs $\A$ on the data and then outputs $\frac{M}{n}\privtheta$. Note that this implies that for every dataset $\D\subseteq\hypcnz^p$ with $\ltwo{\sum_{i=1}^n d_i}\in [M-1, M+1]$, with probability more than $1/2$, $\ltwo{\tilde\A(\D)-q(\D)}\neq \Omega\left(\min\left(1, \frac{p}{\eps n}\right)\right)$ which contradicts Part~1 of Lemma~\ref{lem:1-way-marg-lower}. Thus, there must exist a dataset $\D\subseteq\hypcnz^p$ with $\ltwo{\sum_{i=1}^n d_i}=\Omega\left(\min\left(n, p/\eps\right)\right)$ such that with probability at least $1/2$, we have ${\small\ltwo{\privtheta-\theta^*}= \Omega\left(1\right)}$. Therefore, from the observation we made in the previous paragraph, we have, with probability at least $1/2$, 
$\empL(\privtheta; \D)-\empL(\theta^*; \D)= \Omega\left(\min\left(n, p/\epsilon\right)\right).$
\end{proof}

\begin{thm}[Lower bound for $(\epsilon,\delta)$-differentially private algorithms]
Let $n, p\in\mathbb{N}$, $\epsilon>0$, and $\delta=o(\frac{1}{n})$. For every $(\epsilon,\delta)$-differentially private algorithm (whose output is denoted by $\privtheta$), there is a dataset $\D=\{d_1, \ldots, d_n\} \subseteq \hypcnz^p$ such that, with probability at least $1/3$ (over the algorithm random coins), we must have 
$$\empL(\theta; \D)-\empL(\theta^*; \D)= \Omega\left(\min\left(n, \sqrt{p}/\epsilon\right)\right)$$ 
where $\theta^*=\frac{\sum_{i=1}^nd_i}{\ltwo{\sum_{i=1}^nd_i}}$ is the minimizer of $\empL( . ; \D)$ over $\B$ and $\empL$ is given by (\ref{lin-dec-loss}).
\label{thm:lower-lip-eps-delta}
\end{thm}

\begin{proof}
We use Part~2 of Lemma~\ref{lem:1-way-marg-lower} and follow the same lines of the proof of Theorem~\ref{thm:lower-lip-eps}.
\end{proof}

\mypar{Dependence on $L$ and $\ltwo{\C}$}  Although our lower bounds above are derived for $L=1$ and $\ltwo{\C}=2$, one can easily get their counterparts in the general case, i.e., for arbitrary values of $L$ and $\ltwo{\C}$. The only difference is that the lower bounds for the general case pick up an extra factor of $L\ltwo{C}$. To see this, let $L=\alpha$ and $\ltwo{\C}=2\beta$ for arbitrary $\alpha, \beta>0$. First, we change (inflate or shrink) the parameter set from $\B$ to $\beta\B$ and we change our loss function in (\ref{lin-loss}) to $\tilde\ell(\theta; d)= - \alpha \langle \theta, d\rangle,~\theta\in\beta\B, ~d\in\hypcnz^p.$ Let's denote the corresponding big loss function by $\tilde\eL$. Let $\theta^{\dagger}$ be the minimizer of $\tilde\eL(. ; \D)$. Now, note that $\theta^{\dagger}=\beta\frac{\sum_{i=1}^nd_i}{\ltwo{\sum_{i=1}^nd_i}}=\beta\theta^*$ where $\theta^*$ is the minimizer of our original big loss function $\eL(. ; \D)$ given by (\ref{lin-dec-loss}). Finally, observe that for any $\theta\in\beta\B$ and dataset $\D$, we have $\tilde\eL(\theta; \D)-\tilde\eL(\theta^{\dagger}; \D)= \alpha\beta \left(\empL\left(\frac{\theta}{\beta}; \D\right)-\empL(\theta^*; \D)\right)$. This shows that our bounds above get scaled by $L\ltwo{\C}$ in the general case. Hence, given our upper bounds in Sections~\ref{sec:gradDesc} and \ref{sec:privConvexUpper}, our lower bounds in this section are tight for all values of $L$ and $\ltwo{\C}$.
\subsection{Lower bounds for Strongly Convex Functions}
\label{sec:lipschitzConvexLowerStrong} 

We give here lower bounds on the excess risk of $\epsilon$ and $(\epsilon, \delta)$ differentially private optimization algorithms for the class of strongly convex decomposable loss function $\empL(\theta;\D)$. Let $\ell(\theta;d)$ be half the squared $L_2$-distance between $\theta \in \B$ and $d\in\hypcnz^p$, that is 
\begin{equation}
\ell(\theta;d)=\frac{1}{2}\ltwo{\theta - d}^2\label{quad-loss}
\end{equation}
Note that $\ell$, as defined, is $1$-Lipschitz and $1$-strongly convex. For a dataset $\D=\{d_1, \ldots, d_n\}\subseteq \hypcnz^p$, the decomposable loss function is defined as 
\begin{equation}
\empL(\theta;\D)=\frac{1}{2}\sum_{i=1}^n\ltwo{\theta - d_i}^2\label{quad-dec-loss}
\end{equation} 
Notice that the minimizer of $\eL( . ;\D)$ over $\B$ is $\theta^{\ast}=\frac{1}{n}\sum{d_i}$ which is equal to $q(\D)$ in the terminology of Lemma~\ref{lem:1-way-marg-lower}. Note also that we can write the excess risk as
\begin{equation}
\empL(\privtheta; \D)-\empL(\theta^{\ast})=\frac{n}{2} \ltwo{\privtheta-q(\D)}^2\label{excess-to-l2-sq}
\end{equation}

\begin{thm}[Lower bound for $\epsilon$-differentially private algorithms]
Let $n, p\in\mathbb{N}$ and $\epsilon>0$. For every $\epsilon$-differentially private algorithm (whose output is denoted by $\privtheta$), there is a dataset $\D=\{d_1, \ldots, d_n\} \subseteq \hypcnz^p$ such that, with probability at least $1/2$ (over the algorithm random coins), we must have
$$\empL(\privtheta; \D)-\empL(\theta^{\ast}; \D)=\Omega\left(\min\left(n, \frac{p^2}{\epsilon^2 n}\right)\right)$$
where $\theta^{\ast}=\frac{1}{n}\sum{d_i}$ is the minimizer of $\empL( . ; \D)$ over $\B$ and $\empL$ is given by (\ref{quad-dec-loss}).
\label{thm:lower-str-cnvx-eps}
\end{thm}

\begin{proof}
The proof follows directly from (\ref{excess-to-l2-sq}) and Part~1 of Lemma~\ref{lem:1-way-marg-lower}.
\end{proof}

\begin{thm}[Lower bound for $(\epsilon, \delta)$-differentially private algorithms]
Let $n, p\in\mathbb{N}$, $\epsilon>0$, and $\delta=o(\frac{1}{n})$. For every $(\epsilon,\delta)$-differentially private algorithm (whose output is denoted by $\privtheta$), there is a dataset $\D=\{d_1, \ldots, d_n\} \subseteq \hypcnz^p$ such that, with probability at least $1/3$ (over the algorithm random coins), we must have 
$$\empL(\privtheta; \D)-\empL(\theta^{\ast}; \D)= \Omega\left(\min\left(n, \frac{p}{\epsilon^2 n}\right)\right)$$
where $\theta^{\ast}=\frac{1}{n}\sum{d_i}$ is the minimizer of $\empL( . ; \D)$ over $\B$ and $\empL$ is given by (\ref{quad-dec-loss}).
\label{thm:lower-str-cnvx-eps-delta}
\end{thm}

\begin{proof}
The proof follows directly from (\ref{excess-to-l2-sq}) and Part~2 of Lemma~\ref{lem:1-way-marg-lower}.
\end{proof}

\mypar{Dependence on $L, \ltwo{\C},$ and $\Delta$} Our lower bounds in this section are derived for the case where $L=\Delta=1$ and $\ltwo{\C}=2$. For any values $L, \ltwo{\C},$ and $\Delta$ such that $\frac{\Delta\ltwo{\C}}{L}=\Omega(1)$, these lower bounds pick up an extra factor of $\frac{L^2}{\Delta}$.  To see this, let $L=\alpha$ and $\ltwo{\C}=2\beta$ for arbitrary $\alpha, \beta>0$. First, instead of the data universe $\hypcnz^p$, we choose the dataset entries from $\hypcnzsc^p$ and let any such dataset be denoted by $\tilde\D$. We change the parameter set from $\B$ to $\beta\B$ and change the loss function in (\ref{quad-loss}) to $\tilde\ell(\theta, \tilde{d})=\frac{\alpha}{2\beta}\ltwo{\theta - \tilde{d}}^2, \theta~\in\beta\B,~ \tilde{d}\in\hypcnzsc^p$. Clearly, $\tilde\ell(. , \tilde{d})$ is $\alpha$-Lipschitz and $\frac{\alpha}{\beta}$-strongly convex. Let $\tilde\eL$ denote the corresponding big loss function and $\theta^\dagger$ denote the minimizer of $\tilde\eL(. ; \tilde\D)$. Note that $\theta^{\dagger}=q(\tilde\D)=\beta \theta^*$ where $\theta^*$ is the minimizer of our original big loss function $\eL( ; \D)$ given by (\ref{quad-dec-loss}) for the dataset $\D=\frac{1}{\beta}\tilde\D$. Finally, observe that for any $\theta\in\beta\B$ and dataset $\tilde\D\subseteq\hypcnzsc^p$, we have $\tilde\eL(\theta; \tilde\D)-\tilde\eL(\theta^{\dagger}; \tilde\D)=\alpha\beta \left(\empL\left(\frac{\theta}{\beta}; \D\right)-\empL(\theta^*; \D)\right)$ where $\D=\frac{1}{\beta}\tilde\D\subseteq\hypcnz^p$. This shows that the lower bounds in this case get scaled by $\alpha\beta = \frac{\alpha^2}{\alpha/\beta}= L^2/\Delta$. In fact, this is true as long as $\frac{\Delta\ltwo{\C}}{L}=\Omega(1)$. 

Hence, our upper bounds in Sections~\ref{sec:gradDesc} and \ref{sec:out-pert} imply that our lower bounds are tight for all values of $L, \ltwo{\C},$ and $\Delta$ for which $\frac{\Delta\ltwo{\C}}{L}=\Omega(1)$. In other words, in the general case (where $\frac{\Delta\ltwo{\C}}{L}$ is not necessarily $\Omega(1)$), these lower bounds are tight up to a factor of $\frac{\Delta\ltwo{\C}}{L}$.
\section{Efficient Sampling from Logconcave Distributions over Convex Sets and The Proof of Theorem~\ref{thm:eff-samp-guarantees}}\label{sec:implement-eff-exp}

In this section, we discuss a generic construction of an efficient algorithm for sampling from a logconcave distribution over an arbitrary convex bounded set. Such algorithm gives a multiplicative distance guarantee on the distribution of its output, that is, it outputs a sample from a distribution that is within a constant factor (close to 1) from the desired logconcave distribution. As a by-product of our generic construction, we give the construction of our efficient $\epsilon$-differentially private algorithm $\A_{\sf eff-exp-samp}$ whose construction is outlined in Section~\ref{sec:eff-pure-eps-alg} and prove Theorem~\ref{thm:eff-samp-guarantees}. As argued in Section~\ref{sec:eff-pure-eps-alg}, we will assume that the convex set is already in isotropic position. The reader may refer to Section~\ref{sec:eff-pure-eps-alg} for the details of dealing with the general case (where the set is not necessarily isotropic) and the effect of that on the running time.

We start by the following lemma which describes Algorithm $\A_{\sf cube-samp}$ for sampling from a distribution proprtional to a given logconcave function $F$ defined over a hypercube $A$.

\begin{lem}
Let $A\subset\re^p$ be a $p$-dimensional hypercube with edge length $\tau$. Let $F$ be a logconcave function that is strictly positive over $A$ where $\log F$ is $\eta$-Lipschitz. Let $\mu_{\sf A}$ be the probability measure induced by the density $\frac{F}{\int\limits_{u\in A}F(u)du}$. Let $\tilde{\epsilon} >0$. There is an algorithm $\A_{\sf cube-samp}$ that takes $A,~F,~\eta$ and $\teps$ as inputs, and outputs a sample $\htheta\in A$ that is drawn from a continuous  distribution $\hmuA$ over $A$ with the property that $\dist(\hmuA, \muA)\leq \teps$. Moreover, the running time of $\A_{\sf cube-samp}$ is
$$O\left(\frac{\eta^2\tau^2}{\teps^2}p^3\max\left(p\log\left(\frac{\eta\tau p}{\teps}\right), \eta \tau\right)\right).$$
\label{lem:samp-A}
\end{lem}
\begin{proof}
Let $\gamma=\frac{\teps}{2\eta\sqrt{p}}$. We construct a grid $\G_{\gamma}\triangleq\{u\in\re^p: ~u_j + \frac{\gamma}{2}\text{ is integer multiple of }\gamma, 1\leq j\leq p\}$. Next, we run the grid-walk algorithm of \cite{applegate1991sampling} with the logconcave weight function $F$ on $A\cap\G_{\gamma}$. It follows from the results of \cite{applegate1991sampling} that (i) the grid-walk is a lazy, time-reversible Markov chain, (ii) the stationary distribution of such grid-walk is $\pi=\frac{F}{\sum_{u\in A\cap\G_{\gamma}}F(u)}$, and (iii) the grid-walk has conductance $\phi\geq \frac{\teps}{8\eta\tau p^{\frac{3}{2}}e^{\frac{\teps}{2}}}$. We run the grid-walk for $t_{\infty}$ steps (namely, the $L_{\infty}$ mixing time\footnote{That is, the mixing time w.r.t. the relative distance $\dist$ defined in Section~\ref{sec:eff-pure-eps-alg}} of the walk) and output a sample $\hat{u}\in A\cap\G_{\gamma}$. Then, we uniformly sample a point $\theta$ from the grid cell whose center is $\hat{u}$. Let $\hat{\pi}$ denote the distribution of the output $\hat{u}$ of the grid-walk after $t_{\infty}$ steps. Let $\hmuA$, as in the statement of the lemma, denote the distribution of $\theta$ that is uniformly sampled from the grid cell whose center is $\hat{u}$. Now, suppose that after $t_{\infty}$ steps it is guaranteed to have $\dist(\hat{\pi}, \pi)\leq \frac{\teps}{2}$. Then, since $\log F$ is $\eta$-Lipschitz and $\gamma=\frac{\teps}{2\eta\sqrt{p}}$ (where, as defined above, $\gamma$ is the edge length of every cell of $\G_{\gamma}$), it is easy to show that $\dist(\hmuA, \muA)\leq \teps$. Hence, it remains to show a bound on $t_{\infty}$, the $L_{\infty}$ mixing time of the Markov chain given by the grid-walk. Specifically, $t_{\infty}$ is the number of the steps on the grid-walk required to have $\dist(\hat{\pi}, \pi)\leq \frac{\teps}{2}$. Towards this end, we use the result of \cite{morris2005evolving} on the rapid mixing of lazy Markov chains with finite state space. We formally restate this result in the following lemma.
\begin{lem}[Theorem~1 in \cite{morris2005evolving}]
Let $P$ be a lazy, time reversible Markov chain over a finite state space $\Gamma$. Then, the time $t_{\infty}$ required for relative $L_\infty$ convergence of $\epsilon'$ is at most $1+\int_{4\pi^*}^{4/\epsilon'}\frac{4dx}{x\Phi^2(x)}$. Here, $\Phi(x)=\inf\{\phi_S:\pi(S)\leq x\}$ where $\phi_S$ denotes the conductance of the set $S\subseteq\Gamma$ and $\pi^*$ is the minimum probability assigned by the stationary distribution.
\label{thm:countable}
\end{lem}
Now, setting $\epsilon'=\frac{\teps}{2}$ in the above lemma and using the fact that $\Phi(x)\geq\phi\geq\frac{\teps}{8\eta\tau p^{\frac{3}{2}}e^{\frac{\teps}{2}}}$ for all $x$, we get
$$t_{\infty}= O\left(\frac{\eta^2\tau^2 p^3}{\teps^2}\log\left(\frac{1}{\teps\pi^*}\right)\right).$$ Observe that
$$\pi(u)=\frac{F(u)}{\sum_{v\in A\cap\G_{\gamma}}F(v)}\geq\frac{e^{-\eta \tau}}{\left(\frac{\tau}{\gamma}\right)^p}$$
where the last inequality follows from the fact that $\log F$ is $\eta$-Lipschitz. Plugging the value we set for  $\gamma$, we get $t_{\infty}=O\left(\frac{\eta^2\tau^2}{\teps^2}p^3\max\left(p\log\left(\frac{\eta\tau p}{\teps}\right), \eta \tau\right)\right)$. This completes the proof.
\end{proof}

Now, suppose we are given an arbitrary bounded convex set $\C$ and a logconcave function $F(\theta)=e^{-f(\theta)},~\theta\in\C$ where $f$ is a convex $\eta-$Lipschitz function on $\C$. Having Algorithm~$\A_{\sf cube-samp}$ of Lemma~\ref{lem:samp-A} in hand, we now construct an efficient algorithm $\A_{\sf init-samp}$ that, \emph{with probability at least $1/2$}, outputs a sample in $\C$ from a distribution close (w.r.t. $\dist$) to the distribution on $\C$ proportional to $F$. Our algorithm $\A_{\sf init-samp}$ does this by first enclosing the set $\C$ by a hypercube $A$, then constructing a convex Lipschitz extension of $f$ over $A$ (note that we need this step since $f$ may not be defined outside $\C$). Using a standard trick in literature, our algorithm modulates $F$ by a \emph{guage function} to reduce the weight attributed to points outside $\C$. Namely, given the convex Lipschitz extension of $f$ over $A$, denoted as $\bar{f}$, Algorithm $\A_{\sf init-samp}$ defines a modified function $\tilde{F}(\theta)=e^{-\bar{f}(\theta)-\bar{\psi}_{\alpha}(\theta)}$ where $\bar{\psi}_{\alpha}$ is our guage function with a tuning parameter $\alpha$. The function $\bar{\psi}_{\alpha}$ is chosen such that it is zero inside $\C$ and is montonically increasing outside $\C$ as we move away from $\C$. The exact form of $\bar{\psi}_{\alpha}$ will be given shortly. Algorithm~$\A_{\sf init-samp}$ then calls Algorithm $\A_{\sf cube-samp}$ on inputs $A$ and $\hat{F}$. By appropriately choosing the parameter $\alpha$  of our gauge function, it is guaranteed that, with probability at least $1/2$, $\A_{\sf cube-samp}$ will return a sample in $\C$. Thus, by Lemma~\ref{lem:samp-A}, we reach the desired goal.

The function $\bar{\psi}_{\alpha}(\theta)$ is defined through the Minkowski's norm of $\theta$.The Minkowski's norm of $\theta\in\re^p$ with respect to $\C$, denoted as $\psi(\theta)$, is defined as $\psi(\theta)=\inf\{r>0: \theta\in r\C\}$. We define $\bar{\psi}_{\alpha}(\theta)\triangleq\alpha\cdot\max\{0, \psi(\theta)-1\}$ for some tuning parameter $\alpha>0$ (to be specified later). Note that $\bar{\psi}_{\alpha}(\theta)>0$ if and only if $\theta\notin\C$ since $\mathbf{0}^p\in\C$ (as $\C$ is assumed to be in isotropic position) and $\C$ is convex. Moreover, it is not hard to verify that $\bar{\psi}_{\alpha}$ is $\alpha$-Lipschitz when $\C$ is in isotropic position.

Before we give the precise construction of Algorithm~$\A_{\sf init-samp}$, we discuss first an important step in this algorithm, namely, the construction of convex Lipschitz extension of $f$. The following lemma (which is a variant of Theorem~1 in \cite{lip-ext78}) asserts that any convex Lipschitz efficiently computable function defined over some convex set $\C$ has a Lipschitz extension (with the same Lipschitz constant) over $\re^p$ that is also convex and efficiently computable where the computation efficiency of the extension is granted under the assumption of the existence of some projection oracle.

\begin{lem}[Convex Lipschitz extension, Theorem~1 in \cite{lip-ext78}]
Let $f$ be an efficiently computable, $\eta$-Lipschitz, convex function defined on a convex bounded set $\C\subset \re^{p}$. Then there exists an efficiently computable, $\eta$-Lipschitz convex function $\bar{f}$ defined over $\re^p$ such that $\bar{f}$, restricted to $\C$, is equal to $f$. The efficient computation of $\bar{f}$ is based on the assumption of the existence of a projection oracle.
\label{lem:lipext}
\end{lem}
For clarity and completeness, we give a proof of this lemma here.
\begin{proof}
For the sake of simplicity, let's assume that $\C$ is closed. Actually, this is no loss of generality since we can always redefine $f$ such that it is defined on the closure of $\C$ which is possible because $f$ is continuous on $\C$. We use a standard extension in literature. Namely, define
$$g_{y}(x)\triangleq f(y) + \eta\ltwo{x-y},~ y\in\C, x\in\re^{p}$$
$$\bar{f}(x)=\min\limits_{y\in\C}g_{y}(x), x\in\re^{p}.$$
As a standard result (for example, see \cite{czipser55}), we know that the function $\bar{f}$ on $\re^p$ is $\eta$-Lipschitz extension of $f$. Moreover, since $f$ is convex and $\C$ is a convex set, then for every $x\in\re^{p}$, the computation of $\bar{f}(x)$ is a convex program which can be implemented efficiently using a linear optimization oracle. In particular, a projection oracle would suffice and hence $\bar{f}$ is efficiently computable. It remains to show that $\bar{f}$ is convex. Let $x_1, x_2 \in \re^{p}$. Let $y_1$ and  $y_2$ denote the minimizers of $g_{y}(x_1)$ and $g_{y}(x_2)$ over $y\in\C$, respectively. Let $0\leq \lambda\leq 1$. Define $x_{\lambda}=\lambda x_1 +(1-\lambda)x_2$ and let $y_{\lambda}$ denote the minimzer of $g_{y}(x_{\lambda})$ over $y\in\C$. Now, observe that
\begin{align}
\bar{f}(x_{\lambda})&=g_{y_{\lambda}}(x_{\lambda})\leq g_{\lambda y_1+(1-\lambda)y_2}(x_{\lambda})\nonumber\\
&= f(\lambda y_1 + (1-\lambda) y_2 ) + \eta \ltwo{\lambda (y_1 - x_1) + (1-\lambda). (y_2 - x_2)} \nonumber\\
&\leq \lambda \left( f(y_1) +  \eta \ltwo{y_1 - x_1}\right) + (1-\lambda)\left(f(y_2) + \eta\ltwo{y_2 - x_2}\right)\nonumber\\
&=\lambda \bar{f}(x_1) + (1-\lambda)F(x_2)\nonumber
\end{align}
where the inequality in the first line follows from the fact that $y_{\lambda}$ is the minimzer (w.r.t. $y$) of $g_{y}(x_{\lambda})$ and the inequality in the third line follows from the convexity of $f$ and the $L_2$-norm. This completes the proof of the lemma.
\end{proof}

Now, we give the construction of Algorithm $\A_{\sf init-samp}$ followed by a Lemma asserting the probabilistic guarantee discussed above.

\begin{algorithm}[htb]
	\caption{$\A_{\sf init-samp}$: Efficient Log-Concave Sampling with a Probabilistic Guarantee}
	\begin{algorithmic}[1]
		\REQUIRE A bounded convex set $\C$, a convex function $f$ defined over $\C$, Lipschitz constant $\eta$ of $f$, desired multiplicative distance guarantee $\teps$.
        \STATE Find a cube $A\supseteq\C$ with edge length $\tau=\linf{\C}$.
        \STATE Find a convex Lipschitz extension $\bar{f}(\theta)=\min\limits_{u\in\C}\left(f(u) + \eta \ltwo{\theta-u}\right)$.
        \STATE  $\bar{\psi}_{\alpha}(\theta) = \alpha\cdot\max\{0, \psi(\theta)-1\}$ with $\alpha=3e^{2\teps}(\eta\ltwo{\C} + p)$, where $\psi(\theta)$ is the Minkowski's norm of $\theta$ w.r.t. $\C$ as defined above.
        \STATE $F(\theta)=e^{-\bar{f}(\theta)-\bar{\psi}_{\alpha}(\theta)}$.
         \STATE Output $\htheta=\A_{\sf cube-samp}\left(A, F, \eta+\alpha, \frac{\teps}{2}\right)$\label{step5-alg-init-samp}
	\end{algorithmic}
	\label{Alg:init-samp}
\end{algorithm}

\begin{lem}
With probability at least $1/2$, $\A_{\sf init-samp}$ (Algorithm~\ref{Alg:init-samp} above) outputs $\htheta\in\C$. Moreover, the conditional distribution of $\htheta$ conditioned on the event $\htheta\in\C$ is at multiplicative distance $\dist$ which is at most $\teps$ from the distribution induced by $F$ over $\C$, i.e., within multiplicative distance $\teps$ from the desired distribution $\frac{e^{-f(\theta)}}{\int\limits_{\theta\in\C}e^{-f(\theta)}d\theta},~\theta\in\C$. The running time of $\A_{\sf init-samp}$ is
$$O\left(\frac{\tilde\eta^2\tau^2}{\teps^2}p^3\max\left(p\log\left(\frac{\tilde\eta\tau p}{\teps}\right), \tilde\eta \tau\right)\right)$$
where $\tilde\eta=\max(\eta\ltwo{\C}, p)$.
\label{lem:init-samp}
\end{lem}
\begin{proof}
By Lemma~\ref{lem:samp-A}, we know that $\htheta$ (the output of $\A_{\sf cube-samp}$ in Step~\ref{step5-alg-init-samp}) has a distribution $\hmuA$ with the property that $\dist(\hmuA, \muA)\leq \teps$ where $\muA(u)=\frac{F(u)}{\int\limits_{v\in A}F(v)dv},~u\in A$. We will show that $\int\limits_{\theta\in A\setminus\C}\hmuA(\theta)d\theta\leq\int\limits_{\theta\in\C}\hmuA(\theta)d\theta$. In particular, it suffices to show that $\int\limits_{\theta\in A\setminus\C}F(\theta)d\theta\leq e^{-2\teps}\int\limits_{\theta\in\C}F(\theta)d\theta$. Towards this end, consider a differential ($p$-dimensional) cone with a differential angle $d\omega$ at its vertex which is located at the origin (i.e., inside $\C$ since $\C$ is in isotropic position). Let $\theta_0$ be the point where the axis of the cone intersects with the boundary of $\C$. The set $\C$ divides the cone into two regions; one inside $\C$ and the other is outside $\C$. We now show that, for any such cone, the integral of $F$ over its region outside $\C$, denoted by $\Ico$, is less than the integral of $e^{-2\teps}F$ over the region inside $\C$ which is denoted by $\Ici$. First, observe that
\begin{align}
\Ici&=d\omega~ p\ltwo{\htheta}^p\int_{0}^{1}e^{-\bar{f}(r \theta_0)}r^{p-1}dr\geq d\omega~ p\ltwo{\htheta}^p e^{-\bar{f}(\theta_0)}\int_{0}^{1}e^{-\eta\ltwo{\C} (1-r)}r^{p-1}dr\nonumber\\
&= d\omega~ p\ltwo{\htheta}^p e^{-\bar{f}(\theta_0)}\int_{0}^{1}e^{-\eta\ltwo{\C} r}(1-r)^{p-1}dr\geq d\omega~ p\ltwo{\htheta}^p e^{-\bar{f}(\theta_0)}\int_{0}^{\frac{1}{\eta\ltwo{\C} + p}}
(1-\eta\ltwo{\C} r) (1-pr)dr\nonumber\\
&\geq d\omega~ p\ltwo{\htheta}^p e^{-\bar{f}(\theta_0)}\int_{0}^{\frac{1}{\eta\ltwo{\C} + p}}
\left(1-\left(\eta\ltwo{\C}+p\right) r\right)dr= d\omega~ p\ltwo{\htheta}^p e^{-\bar{f}(\theta_0)}\frac{1}{2(\eta\ltwo{\C} + p)} \nonumber
\end{align}
where the second inequality in the first line follows from the Lipschitz property of $\bar{f}$ and the second inequality in the second line follows from the fact that $e^{-x}\geq 1-x$ and $(1-x)^{p-1}\geq 1-px$.
On the other hand, we can upper bound $\Ico$ as follows.
\begin{align}
\Ico&\leq d\omega~ p\ltwo{\htheta}^p\int_{1}^{\infty}e^{-\bar{f}(r \theta_0)}e^{-\alpha (r-1)}r^{p-1}dr\leq d\omega~ p\ltwo{\htheta}^p e^{-\bar{f}(\theta_0)}\int_{1}^{\infty}e^{\eta\ltwo{\C} (r-1)}e^{-\alpha(r-1)}r^{p-1}dr\nonumber\\
&\leq 2(\eta\ltwo{\C} + p)\Ici \int_{0}^{\infty}e^{-\left(\alpha - \left(\eta\ltwo{\C}+ p\right)\right) r}\leq e^{-2\teps}\Ici
\end{align}
where the last inequality follows from the setting of $\alpha$ we made in Algorithm~\ref{Alg:init-samp}. Since this is true for any differential cone as described above, this proves that $\A_{\sf init-samp}$ outputs $\htheta\in\C$ with probability at least $1/2$.

Next, let $\good$ denote the event that $\htheta\in\C$ and let $\hmug$ denote the conditional distribution of $\htheta$ conditioned on $\good$. Let $\muC$ denote the distribution induced by $F$ on $\C$, that is, $\frac{F}{\int_{\theta\in\C}F(\theta)d\theta}$. Observe that, for any measurable set  $\mU\subseteq\C$, $\hmug(\mU)=\frac{\hmuA(\mU)}{\hmuA(\C)}$. Now, since $\muC(\mU)=\frac{\muA(\mU)}{\muA(\C)}$, by Lemma~\ref{lem:samp-A}, we have $\dist\left(\hmug, \muC\right)\leq \teps$.

Finally, regarding the running time of $\A_{\sf init-samp}$, note that, as pointed out at the end of Section~\ref{sec:eff-pure-eps-alg}, we assume that $f$ is efficiently computable and that there exist a membership oracle (to efficiently test the membership of a point w.r.t. $\C$) and a projection oracle (to efficiently construct the convex Lipschitz extension $\bar{f}$). This enables us to efficiently implement the first four steps of $\A_{\sf init-samp}$. However, we do not take into account the extra polynomial factor in running time that is required to perform those steps since it would be highly dependent on the specific structure of $\C$. Thus, under this assumption, the running time of $\A_{\sf init-samp}$ is the same as the running time of $\A_{\sf cube-samp}$ with inputs $A, ~F, ~\eta+\alpha, $ and $\frac{\teps}{2}$. The expression in the lemma follows directly from the running time of $\A_{\sf cube-samp}$ in Lemma~\ref{lem:samp-A} and the fact that $\eta+\alpha=O\left(\max(\eta\ltwo{\C}, ~p)\right)$.
\end{proof}

Now, using a standard boosting approach, we construct an algorithm $\A_{\sf eff-samp}$ that outputs a sample $\theta\in\C$ with probability $1$ whose distribution is at multiplicative distance at most $\teps$ from the desired distibution on $\C$.

\begin{algorithm}[htb]
	\caption{$\A_{\sf eff-samp}$: Efficient Log-Concave Sampling over a Convex Set}
	\begin{algorithmic}[1]
		\REQUIRE A bounded convex set $\C$, a convex function $f$ defined over $\C$, Lipschitz constant $\eta$ of $f$, desired multiplicative distance guarantee $\teps$.
		\STATE Find $\tau=\linf{\C}$.
		\STATE Set $m = 4\eta\ltwo{\C} + p\log(\ltwo{\C}) + \log\left(\frac{1}{1-e^{-\frac{\teps}{4}}}\right)$.
       \FOR {$1\leq i\leq m~$}
		\STATE $\htheta = \A_{\sf init-samp}(\C, f, \eta, \frac{\teps}{4})$.
		\IF {$\htheta\in\C$}
			\STATE Output $\htheta$ and \textbf{abort}.
		\ENDIF
	\ENDFOR
	\STATE Output a unifromly random sample $\htheta$ from the unit ball $\B$. (Note that $\B\subseteq\C$ since $\C$ is in isotropic position.)\label{step7-Alg-eff-samp}
	\end{algorithmic}
	\label{Alg:eff-samp}
\end{algorithm}

\begin{lem}
Let $\hmuC$ denote the distribution of $\htheta$ (the output of Algorithm $\A_{\sf eff-samp}$) and $\muC$ denote the desired distribution $\frac{e^{-f}}{\int\limits_{\theta\in\C}e^{-f(\theta)}d\theta}$ on $\C$. We have $\dist(\hmuC, \muC)\leq\teps$. Moreover, the running time of $\A_{\sf eff-samp}$ is
$$O\left(\frac{\tilde\eta^2\tau^2}{\teps^2}p^3\cdot\max\left(p\log\left(\frac{\tilde\eta\tau p}{\teps}\right), \tilde\eta \tau\right)\cdot\max\left(p\log(\ltwo{\C}), \eta\ltwo{\C}, \log\left(\frac{1}{\teps}\right)\right)\right)$$
where  $\tilde\eta=\max(\eta\ltwo{\C}, p)$.
\label{lem:eff-logconc-samp}
\end{lem}
\begin{proof}
Let $\hmug$ denote the conditional distribution of $\htheta$ (the output of Algorithm~$\A_{\sf eff-samp}$) conditioned on the event that $\A_{\sf init-samp}$ outputs a sample in $\C$ in one of the $m$ iterations of the \textbf{for} loop.  From Lemma~\ref{lem:init-samp}, it is easy to see that the probability measure of the output of $\A_{\sf eff-samp}$ can be expressed as
$$d\hmuC(\htheta)=(1-2^{-m})d\hmug(\htheta)+2^{-m}\cdot\frac{1}{\sf Vol(\B)}\cdot\mathbf{1}(\htheta\in\B)$$
where $\mathbf{1}( . )$ is the standard indicator function, i.e., it takes value $1$ whenever $\htheta\in\B$ and zero otherwise. Also, from Lemma~\ref{lem:init-samp}, we know that $\dist(\hmug, \muC)\leq\frac{\teps}{4}$. Let $\mu^*$ denote the minimum value of the density function $\frac{e^{-f(\theta)}}{\int\limits_{u\in\C}e^{-f(u)}du}$ for $\theta\in\C$. By the Lipschitz property of $f$, we have
$$\mu^*\geq \frac{e^{-2\eta\tau}}{{\sf Vol}(\C)}=\frac{1}{\sf Vol(\B)}\frac{\sf Vol(\B)}{\sf Vol(\C)} e^{-2\eta\ltwo{\C}}=\frac{1}{\sf Vol(\B)}\frac{e^{-2\eta\ltwo{\C}}}{\ltwo{\C}^p}.$$
Hence, our choice of $m$  guarantees that
$$\frac{2^{-m}}{\mu^*{\sf Vol(\B)}}\leq e^{\frac{\teps}{2}}\left(e^{\frac{\teps}{2}}-1\right).$$
It also guarantees that $(1-2^{-m})\geq e^{-\frac{\teps}{4}}$. Putting this together, we get
$$e^{\dist(\hmuC,~ \muC)}\leq e^{\frac{\teps}{4}}~e^{\dist(\hmug, ~\muC)}+\frac{2^{-m}}{\mu^*{\sf Vol(\B)}}\leq e^{\teps}.$$

The running time of $\A_{\sf eff-samp}$ is at most $O(m \cdot T_{\A_{\sf init-samp}})$ where $T_{\A_{\sf init-samp}}$ is the running time of $\A_{\sf init-samp}$ (which is of the same order as that of $\A_{\sf cube-samp}$ given in Lemma~\ref{lem:samp-A}). Note that Step~\ref{step7-Alg-eff-samp} can be carried out in linear time using standard methods in literature. Finally, by plugging in our choice for the value of $m$ gives the expression in the lemma statement. This completes the proof.
\end{proof}

\subsection{Efficient $\epsilon$-Differentially Private Algorithm for Lipschitz Convex Loss}

In this section, we show a straightforward construction for our efficient Algorithm $\A_{\sf eff-exp-samp}$ (referred to in  Section~\ref{sec:eff-pure-eps-alg}) based on the construction established above for efficient logconcave sampling. Based on the results established above in this section, we give a proof of Theorem~\ref{thm:eff-samp-guarantees} which will also be fairly straightforward. First, fix a dataset $\D$. Our goal is to construct an efficient version of Algorithm $\A_{\sf exp-samp}$ (Algorithm~\ref{Alg:GenSamp} from Section~\ref{sec:lipschitzConvex}). To do this, we simply run Algorithm  $\A_{\sf eff-samp}$ (Algorithm~\ref{Alg:eff-samp} above) with the function $f$ instantiated with the scaled decomposable loss function $\frac{\epsilon}{6L\ltwo{\C}}\eL( . ;\D)$ defined over the convex bounded set $\C$ (which is assumed to be in isotropic position as discussed in Section~\ref{sec:eff-pure-eps-alg}). Hence, $\eta$ in our case is $\frac{n\epsilon}{6\ltwo{\C}}$. Namely, as shown below, our $\epsilon$-differentially private algorithm $\A_{\sf eff-exp-samp}$ is an instantiation of $\A_{\sf eff-samp}$ on inputs $\C,~ \frac{\epsilon}{6L\ltwo{\C}}\eL( . ;\D),~ \frac{n\epsilon}{6\ltwo{\C}},$ and $\frac{\epsilon}{3}$.

\begin{algorithm}[htb]
	\caption{$\A_{\sf eff-exp-samp}$: Efficient Log-Concave Sampling over a Convex Set}
	\begin{algorithmic}[1]
		\REQUIRE A dataset $\D$ of size $n$, a bounded convex set $\C$, convex $L-$Lipschitz loss function $\ell$, and a privacy parameter $\epsilon$.
		\STATE $\empL(\theta;\D)=\sum\limits_{i=1}^n\ell(\theta;d_i)$.
		\STATE Output $\privtheta=\A_{\sf eff-samp}\left(\C,~\frac{\epsilon}{6L\ltwo{\C}}\eL( . ;\D),~ \frac{n\epsilon}{6\ltwo{\C}},~\frac{\epsilon}{3}\right)$.
\end{algorithmic}
	\label{Alg:eff-exp-samp}
\end{algorithm}

The choice of the scaling factor $\frac{\epsilon}{6L\ltwo{\C}}$ of the loss function and the multiplicative distance guarantee of $\frac{\epsilon}{3}$ is tuned to yield an $\epsilon$-differentially private algorithm. To see this, we rely on the following simple lemma given in \cite{HT09}.

\begin{lem}[follows from Lemma~A.1 in \cite{HT09}]
Let $\epsilon, \teps >0$. Let $\Q\subseteq\re^p$. For every dataset $\D$, let $\mu^{\D}$ denote the distribution (over $\Q$) of the output of an $\epsilon$-differentially private algorithm $\A_1$ when run on the input dataset $\D$, and $\hat{\mu}^{\D}$ be the distribution (over $\Q$) of the output of some algorithm $\A_2$ when run on $\D$. Suppose that $\dist(\hat{\mu}^{\D}, \mu^{\D})\leq \teps$ for all $\D$. Then, $\A_2$ is $(2\teps+\epsilon)$-differentially private.
\label{lem:diff-priv-from-close-dist}
\end{lem}

\boldpar{Proof of Theorem~\ref{thm:eff-samp-guarantees}} Having Lemmas~\ref{lem:eff-logconc-samp} and \ref{lem:diff-priv-from-close-dist} in hand, the proof of Theorem~\ref{thm:eff-samp-guarantees} becomes straightforward. First, we show differential privacy of Algorithm~\ref{Alg:eff-exp-samp}. For any dataset $\D$, let $\mu^{\D}$  be the distribution of $\theta$ proportional to $e^{- \frac{\epsilon}{6L\ltwo{\C}}\eL(\theta ;\D)}$. Note that $\mu^{\D}$ is the distribution of the output of Algorithm $\A_{\sf exp-samp}$ (Algorithm~\ref{Alg:GenSamp} from Section~\ref{sec:lipschitzConvex}) when $\epsilon$ is replaced with $\frac{\epsilon}{3}$. Let $\hmu^{\D}$  be the distribution of the output $\privtheta$ of $\A_{\sf eff-exp-samp}$ (Algorithm~\ref{Alg:eff-exp-samp} above). From Lemma~\ref{lem:eff-logconc-samp}, it follows that $\dist\left(\hmu^{\D}, \mu^{\D}\right)\leq \frac{\epsilon}{3}$. Hence, from Theorem~\ref{thm:privVol} and Lemma~\ref{lem:diff-priv-from-close-dist}, we reach the fact that $\A_{\sf eff-exp-samp}$ is $\epsilon$-differentially private.

To show the utility guarantee of $\A_{\sf eff-exp-samp}$, we first observe that the distribution of the output $\privtheta$ is close with respect to $\dist$ to (i.e., within a constant factor of) the distribution of the output of $\A_{\sf exp-samp}$ (Algorithm~\ref{Alg:GenSamp} from Section~\ref{sec:lipschitzConvex}), and hence, the utility analysis follows the same lines of Theorem~\ref{thm:abcd2}.

Finally, observe that the running time of $\A_{\sf eff-exp-samp}$ is the same as the running time of $\A_{\sf eff-samp}$ with $\eta$ replaced with $\frac{\epsilon n}{6\ltwo{\C}}$ and $\teps$ replaced with $\frac{\epsilon}{3}$. Also observe that $\tau=\linf{\C}\leq\ltwo{\C}$ and, as a standard assumption, $n=\omega(p)$. Putting this together, we get the running time given in the statement of the theorem. This completes the proof of Theorem~\ref{thm:eff-samp-guarantees}.

\section*{Acknowledgments}

We are grateful to Santosh Vempala and Ravi Kannan for discussions
about efficient sampling algorithms for log-concave distributions over
convex bodies. In particular, Ravi suggested the idea of using a
penalty term to reduce from sampling over $\C$ to sampling over the cube.

\bibliographystyle{plainnat}
\bibliography{reference}
\appendix
\section{Straightforward Smoothing Does Not Yield Optimal Algorithms}
\label{sec:huberization}

In Section \ref{sec:SmoothFn} we saw that the objective perturbation algorithm \eqref{eq:a1Obj} of \cite{CMS11,KST12} already matches the optimal excess risk bounds  for Lipschitz, and Lipschitz and strongly convex functions when the loss function $\ell$ is twice-continuously differentiable with a bounded double derivative $\beta$. A natural question that arises, \emph{is it possible to smoothen out a non-smooth loss function by convolving with a smooth kernel (like the Gaussian kernel) or by Huberization, and still achieve the optimal excess risk bound?} In this section we look at a simple loss functions (the hinge loss) and a very popular Huberization method (quadratic smoothing) argue that there is an inherent cost due to smoothing which will not allow one to get the optimal excess risk bounds.

Consider the loss function $\ell(\theta;d)=(y-x\theta)^+$, where the data point $d=(x,y)$, and $x,y\in[-1,1]$ and $\theta\in\re$. Here the function $f(z)=(z)^+$ is equal to $z$ when $z>0$ and zero otherwise. Clearly $f(z)$ has a point of non-differentiability at zero. We can modify the function $f$, in the following way, to ensure that the resulting function $\hat f$ is smooth (or twice-continuously differentiable). Define $\hat f(z)=f(z)$, when $z<-h$ or when $z>h$. In the range $[-h,h]$, we set $\hat f(z)=\frac{z^2}{4h}+\frac{z}{2}+\frac{h}{4}$. It is not hard to verify that the function $\hat f(z)$ is twice-continuously differentiable everywhere. This form of smoothing is commonly called Huberization. Let the smoothed version of $\ell(\theta;d)$ be defined as $\hat \ell(\theta;d)=\hat f((y-x\theta))$ for $d=(x,y)$.

With the choice of loss function $\hat\ell$, the objective perturbation algorithm is as below. (The regularization coefficient is chosen to ensure that it is at least $\frac{\beta}{2\epsilon}$, where $\beta$ is the smoothness parameter of $\hat\ell$.):
\begin{equation}
\privtheta=\arg\min\limits_{\theta\in[-2,2]}\sum\limits_{i=1}^n\hat\ell(\theta;d_i)+\frac{\theta^2}{8\epsilon h}+b\theta
\label{eq:a1Obj1}
\end{equation}
In \eqref{eq:a1Obj1} the noise $b\sim\mathcal{N}(0,\frac{8\log(1/\delta)}{\epsilon^2})$. In the results to follow, we show that for any choice of the Huberization parameter $h$, there exists data sets of size $n$ from the domain above where the excess risk for objective perturbation will be provably worse than our results in this paper. We present the results for the $(\epsilon,\delta)$-differential privacy case, but the same conclusions hold for the pure $\epsilon$-differential privacy case.

\begin{thm}
  For every $h>0$, there exists $\D$ such the excess risk for the objective perturbation algorithm in \eqref{eq:a1Obj1} satisfies:
  $$\E\left[\empL(\privtheta; \D)-\empL(\theta^*; \D)\right] = \Omega\left(\min \left\{n, \max\{nh, \tfrac 1 h\}\right\}\right) = \Omega(\sqrt{n})\,.$$
  Here the loss function $\empL(\theta;\D)=\sum\limits_{i=1}^n\ell(\theta;d_i)$ (where $\D=\{d_1,\cdots,d_n\}$) and $\theta^*=\arg\min\limits_{\theta\in[-2,2]}\sum\limits_{i=1}^n\ell(\theta;d_i)$.
  \label{thm:hbr}
\end{thm}

\begin{proof}

Consider the data set $\D_1$ with $\frac{n}{3}$ entries being $(x=-1,y=1)$ and $\frac{2n}{3}$ entries being $(x=1,y=-1)$. In the following lemma we lower bound the excess risk on $\D_1$ for a given huberization parameter $h$.

\begin{lem}
Let $\epsilon,\delta$ be the privacy parameters with $\epsilon$ being a constant ($< 1$) and $\delta=\Omega\left(\frac{1}{n^{4}}\right)$. For the data set $\D_1$ mentioned above, the excess risk for objective perturbation \eqref{eq:a1Obj1} is as follows. For all $h>0$, we have $$\E\left[\empL(\privtheta; \D_1)-\empL(\theta; \D_1)\right]=\Omega\left(n\cdot\min\{1,h\}\right).$$
Here the loss function $\empL(\theta;\D_1)=\sum\limits_{i=1}^n\ell(\theta;d_i)$ (where $\D_1=\{d_1,\cdots,d_n\}$) and $\theta^*=\arg\min\limits_{\theta\in[-2,2]}\sum\limits_{i=1}^n\ell(\theta;d_i)$.
\label{thm:huber2}
\end{lem}
\begin{proof}
For the ease of notation, let $\hat\empL(\theta;\D_1)=\sum\limits_{i=1}^n\hat\ell(\theta;d_i)$. First notice two properties of $\hat\empL$: i) the minimizer (call it $\hat\theta$ within the set $[-2,2]$ is at $\max\left\{-2,1-h\right\}$, and ii) $\hat\empL$ is quadratic within the range $[1-h,1+h]$ with strong convexity parameter at least $\frac{n}{6h}$. Additionally notice that $\theta^*=1$ and the regularizer $\frac{\theta^2}{8\epsilon h}$ in \eqref{eq:a1Obj1} is centered at zero. Also by Markov's inequality, w.p. $\geq 2/3$, we have $|b|\leq \frac{8\sqrt{\log(1/\delta)}}{\epsilon}$. Now to satisfy optimality, $\frac{n|\privtheta-\hat\theta|}{3h}\leq |b|$. This suggests that $|\privtheta-\hat\theta|\leq \frac{3h |b|}{n}$.Therefore, the difference $\left(\theta^*-\privtheta\right)$ is at least $\min\left\{1,h\left(1-\frac{3|b|}{n}\right)\right\}$. Therefore the excess risk with probability at least $2/3$ is $\Omega\left(n\cdot\min\{1,h\}\right)$, which concludes the proof.

\end{proof}

Consider a data set $\D_2$ which has exactly $\max\{\frac{n}{2}-\frac{1}{32h},0\}$ entries with $(x=-1,y=1)$ and $\min\{\frac{n}{2}+\frac{1}{32h},n\}$ entries with $(x=1,y=1)$. In the following lemma we lower bound the excess risk on $\D_2$ for a given huberization parameter $h$.

\begin{lem}
Let $\epsilon,\delta$ be the privacy parameters with $\epsilon$ being a constant ($<1$) and $\delta=\Omega\left(\frac{1}{n^{4}}\right)$. Let $h<\frac{1}{\log n}$ be a fixed Huberization parameter. Then for the data set $\D_2$ mentioned above, the excess risk for objective perturbation \eqref{eq:a1Obj1} is as follows.
$$\E\left[\empL(\privtheta; \D_2)-\empL(\theta^*; \D_2)\right]=\Omega\left(\min\left\{\frac{1}{h},n\right\}\right).$$
Here the loss function $\empL(\theta;\D_2)=\sum\limits_{i=1}^n\ell(\theta;d_i)$ and $\theta^*=\arg\min\limits_{\theta\in[-2,2]}\sum\limits_{i=1}^n\ell(\theta;d_i)$.
\label{thm:huber1}
\end{lem}

\begin{proof}
For the ease of notation, let $\hat\empL(\theta;\D_2)=\sum\limits_{i=1}^n\hat\ell(\theta;d_i)$. Notice that within the range $[-1+h,1-h]$, the slope of $\hat\empL(\theta;\D_2)$ is $\max\{\frac{-1}{16h},-n\}$. By the optimality condition of $\privtheta$, we have the following. 
\begin{equation}
\frac{\privtheta}{4\epsilon h}+b-\min\{\frac{1}{16h},n\}=0
\label{eq:smo112}
\end{equation}
Solving for $\privtheta$, we have $\privtheta=\min\left\{\frac{\epsilon}{4},4\epsilon n h\right\}+4b\epsilon h$. By assumption $h< 1/\log n$ and w.p. $\geq 2/3$ we have $|b|\leq \frac{8\sqrt{\log(1/\delta)}}{\epsilon}$. Therefore, w.p. $\geq 2/3$, we have $\privtheta\leq \epsilon$.

Now notice that with the original loss function $\ell$,
$\arg\min\limits_{\theta\in[-2,2]}\sum\limits_{i=1}^n\ell(\theta;d_i)=1$. Since the loss function $\empL(\theta;\D_2)$ has a slope of $\max\{\frac{-1}{16h},-n\}$ in the range $[-1,1]$, the excess risk is $\Omega((1-\epsilon)\min\left\{\frac{1}{h},n\right\})$ which concludes the proof.
\end{proof}
Finally combining Lemmas \ref{thm:huber1} and \ref{thm:huber2} completes the proof of Theorem \ref{thm:hbr}.
\end{proof}

\section{Localization and $(\epsilon, \delta)$-Differentially Private Algorithms for Lipschitz, Strongly Convex Loss}
\label{sec:localization-eps-delta}

We use slightly different version of $\A^{\epsilon}_{\sf out-pert}$ (Algorithm~\ref{Alg:OutPert}) which we denote by $\A^{(\epsilon,\delta)}_{\sf out-pert}$ where the algorithm takes as input an extra privacy parameter $\delta$, it samples the noise vector $b$ from the Gaussian distribution $\mathcal{N}\left(0,\mathbb{I}_{p}\sigma_0^2\right)$ where $\sigma_0^2=4\frac{L^2\log\left(\frac{1}{\delta}\right)}{\Delta^2\epsilon^2 n^2}$, and outputs $\C_0=\{\theta\in\C: \ltwo{\theta-\theta_0}\leq \zeta\sigma_0\sqrt{p}\}$.

Let $\A^{(\epsilon,\delta)}_{\sf gen-Lip}$ denote any generic $(\epsilon,\delta)$-differentially private algorithm for optimizing a decomposable loss of  convex Lipschitz functions over some arbitrary convex set $\tilde{\C}\subseteq\C$. Algorithm~\ref{Algo:GradDesc} from Section~\ref{sec:gradDesc} is an example of $\A^{(\epsilon,\delta)}_{\sf gen-Lip}$. Now, we construct an algorithm $\A^{(\epsilon,\delta)}_{\sf gen-str-convex}$ which is the $(\epsilon,\delta)$ analog of $\A^{\epsilon}_{\sf gen-str-convex}$ (Algorithm~\ref{Alg:eps-gen-st-cnvx}). Namely,  $\A^{(\epsilon,\delta)}_{\sf gen-str-convex}$ runs in similar fashion to $\A^{\epsilon}_{\sf gen-str-convex}$ where the only difference is that it takes an extra privacy parameter $\delta$ as input and calls algorithms  $\A^{(\frac{\epsilon}{2},\frac{\delta}{2})}_{\sf out-pert}$ and $\A^{(\frac{\epsilon}{2},\frac{\delta}{2})}_{\sf gen-Lip}$ instead of $\A^{\frac{\epsilon}{2}}_{\sf out-pert}$ and $\A^{\frac{\epsilon}{2}}_{\sf gen-Lip}$, respectively.

\begin{thm}[Privacy guarantee]
Algorithm $\A^{(\epsilon,\delta)}_{\sf gen-str-convex}$ is $(\epsilon, \delta)$-differentially private.
\label{thm:priv_eps_delta_Gen_st_cnvx}
\end{thm}

\begin{proof}
The privacy guarantee follows directly from the composition theorem together with the fact that $\A^{(\frac{\epsilon}{2},\frac{\delta}{2})}_{\sf out-pert}$ is $(\frac{\epsilon}{2},\frac{\delta}{2})$-differentially private and that $\A^{(\frac{\epsilon}{2},\frac{\delta}{2})}_{\sf gen-Lip}$ is $(\frac{\epsilon}{2},\frac{\delta}{2})$-differentially private by assumption.
 \end{proof}

\begin{thm}[Generic utility guarantee]
Let $\htheta$ denote the output of Algorithm $\A^{(\epsilon,\delta)}_{\sf gen-Lip}$ on inputs $n, \D, \ell, \epsilon, \delta, \tilde{\C}$ (for an arbitrary convex set $\tilde{\C}\subseteq\C$). Let $\nptheta$ denote the minimizer of $\empL(. ; \D)$ over $\tilde{\C}$. If
$$E\left[\empL(\htheta; \D)-\empL(\nptheta; \D)\right]\leq F\left(p, n, \epsilon, \delta, L,\ltwo{\tilde{\C}}\right)$$
for some function $F$, then the output $\privtheta$ of $\A^{(\epsilon,\delta)}_{\sf gen-str-convex}$  satisfies
$$E\left[\empL(\privtheta; \D)-\empL(\theta^*; \D)\right]\leq O\left(F\left(p, n, \epsilon, \delta, L, O\left(\frac{L\sqrt{p\log\left(\frac{1}{\delta}\right)\log(n)}}{\Delta\epsilon n}\right)\right)\right).$$
\label{thm:utility_eps_delta_Gen_st_cnvx}
\end{thm}
\begin{proof}
The proof follows the same lines of the proof of Theorem~\ref{thm:utility_eps_Gen_st_cnvx}  except for the fact that, in Algorithm $\A^{(\frac{\epsilon}{2},\frac{\delta}{2})}_{\sf out-pert}$, the noise vector $b$ is Gaussian and hence using the standard bounds on the norm of an i.i.d. Gaussian vector, we have
$$\Pr\bigg[ \ltwo{b}\leq \zeta\sigma_0\sqrt{p} \bigg]=\Pr\left[\ltwo{b}\leq \zeta\frac{4L\sqrt{\log\left(\frac{2}{\delta}\right)}}{\Delta\epsilon n}\right]\geq 1-e^{-\Omega(\zeta^2)}$$
We set $\zeta=\sqrt{3\log(n)}$ and the rest of the proof follows in the same way as the proof of Theorem~\ref{thm:utility_eps_Gen_st_cnvx}.
\end{proof}

\section{Proof of Lemma~\ref{lem:1-way-marg-lower}}
\label{app:proof-lower-bounds-lem}

\subsection{Proof of Part~1}
We restate Part~1 of the lemma here for convenience.

\vspace{0.5cm}

\boldpar{Lemma~\ref{lem:1-way-marg-lower}- Part~1} \textit {Let $n, p\in\mathbb{N}$ and $\epsilon > 0$. There is a number $M=\Omega\left(\min\left( n, p/\epsilon\right)\right)$ such that for every $\eps$-differentially private algorithm $\A$, there is a dataset $\D=\{d_1, \ldots, d_n\} \subseteq \hypcnz^p$ with $\ltwo{\sum_{i=1}^nd_i}\in [M-1, M+1]$ such that, with probability at least $1/2$ (taken over the algorithm random coins), we have }
$$\ltwo{\A(\D)-q(\D)}=\Omega\left(\min\left(1, \frac{p}{\epsilon n}\right)\right)$$
\textit{where $q(\D)=\frac{1}{n}\sum_{i=1}^{n}d_i$.}

\begin{proof}
We use a standard packing argument. Variants of such argument have appeared in several places in literature, e.g., \cite{HT09} and \cite{De11}. We first construct $K=2^{p/2}$ points $d^{(1)}, ..., d^{(K)}$ in $\hypcnz^p$ such that for every distinct pair $d^{(i)}, d^{(j)}$ of these points, we have
$$\ltwo{d^{(i)}-d^{(j)}}\geq \frac{1}{8}.$$
It is easy to show the existence of such set of points using the probabilistic method (for example, the Gilbert-Varshamov construction of a linear random $(2^{p/2}, p)$-binary code over $\hypcnz$ achieves this property).

Fix $\eps>0$. Define $n^*=\frac{1}{20}\frac{p}{\eps}$. Let's first consider the case where $n\leq n^*$. We construct $K$ datasets $\D^{(1)}, ..., \D^{(K)}$ where for each $i\in[K]$, $\D^{(i)}$ contains n copies of $d^{(i)}$. Note that for all $i\neq j$,
\begin{equation}
\ltwo{q\left(\D^{(i)}\right)-q\left(\D^{(j)}\right)}\geq \frac{1}{8}\label{a}
\end{equation}
Let $\A$ be any $\eps$-differentially private algorithm for answering $q$. Suppose that for every $\D^{(i)}, i\in[K]$, with probability at least $1/2$, $\ltwo{\A\left(\D^{(i)}\right)-q\left(\D^{(i)}\right)}<\frac{1}{16}$, i.e., for every $ i\in[K]$, $\Pr\left[\A\left(\D^{(i)}\right)\in \Bc\left(\D^{(i)}\right)\right]\geq\frac{1}{2}$ where for any dataset $\D$, $\Bc(\D)$ is defined as
\begin{equation}
\Bc(\D)=\{\theta\in\re^p:~ \ltwo{\theta-q(\D)}<\frac{1}{16}\}\label{b}
\end{equation}
Note that for all $i\neq j$, $\D^{(i)}$ and $\D^{(j)}$ differ in all their $n$ entries. Since $\A$ is $\eps$-differentially private, for all $i\in[K]$, we have $\Pr\left[\A\left(\D^{(1)}\right)\in \Bc\left(\D^{(i)}\right)\right]\geq\frac{1}{2}e^{-\eps n}$. Since, by (\ref{a}) abd (\ref{b}), all $\Bc\left(\D^{(i)}\right),~i\in[K],$ are mutually disjoint, then
$$K\cdot\frac{1}{2}e^{-\eps n}\leq \sum_{i=1}^K \Pr\left[\A\left(\D^{(1)}\right)\in \Bc\left(\D^{(i)}\right)\right]\leq 1$$
which implies that $n>n^*$  for sufficiently large $p$ which is a contradiction to the fact that $n\leq n^*$. Hence, there must exist a dataset $\D^{(i)}$ for some $i\in[K]$ on which $\A$ makes an $L_2$-error which is at least $\frac{1}{16}$ with probability at least $\frac{1}{2}$. Note also that the $L_2$ norm of the sum of the entries of such $\D^{(i)}$ is $n$.

Next, we consider the case where $n>n^*$. Fix an arbitrary point $\mathbf{c}\in\hypcnz^p$. As before, we construct $K=2^{p/2}$ datasets $\tilde\D^{(1)}, ..., \tilde\D^{(K)}$ of size $n$ where for every $i\in[K]$, the first $n^*$ entries of each dataset $\tilde\D^{(i)}$ are the same as dataset $\D^{(i)}$ from before whereas the remaining $n-n^*$ entries are constructed as follows. The first $\lceil \frac{n-n^*}{2}\rceil$ of those entries are all copies of $\mathbf{c}$ whereas the last $\lfloor\frac{n-n^*}{2}\rfloor$ are copies of $-\mathbf{c}$. Note that any two distinct datasets $\tilde\D^{(i)}, \tilde\D^{(j)}$ in this collection differ in exactly $n^*$ entries. Let $\A$ be any $\eps$-differentially private algorithm for answering $q$. Suppose that for every $i\in[K]$, with probability at least $1/2$, we have
\begin{equation}
\ltwo{\A\left(\tilde\D^{(i)}\right)-q\left(\tilde\D^{(i)}\right)}<\frac{1}{16}\frac{n^*}{n}\label{c}
\end{equation}
Note that for all $i\in[K]$, $q\left(\tilde\D^{(i)}\right)=\frac{n^*}{n}q\left(\D^{(i)}\right)+\alpha$ where $\alpha=\frac{\mathbf{c}}{n}$ if $n-n^*$ is odd and $0$ if $n-n^*$ is even. Now, we define an algorithm $\hat\A$ for answering $q$ on datasets $\D$ of size $n^*$ as follows. First, $\hat\A$ appends $\lceil\frac{n-n^*}{2}\rceil$ copies of $\mathbf{c}$ followed by $\lfloor\frac{n-n^*}{2}\rfloor$ copies of $-\mathbf{c}$ to $\D$ to get a dataset $\tilde\D$ of size $n$. Then, it runs $\A$ on $\tilde\D$ and outputs $\frac{n}{n^*}\left(\A(\tilde\D)-\alpha\right)$. Hence, by the post-processing propertry of differential privacy, $\hat\A$ is $\eps$-differentially private since $\A$ is $\eps$-differentially private. Thus, assumption (\ref{c}) implies that for every $i\in[K]$, with probability at least $1/2$, we have $\ltwo{\hat\A\left(\D^{(i)}\right)-q\left(\D^{(i)}\right)}<\frac{1}{16}$. However, this contradicts our result in the first part of the proof. Therefore, there must exist a dataset $\tilde\D^{(i)}$ in the above collection such that, with probability at least $1/2$, $$\ltwo{\A\left(\tilde\D^{(i)}\right)-q\left(\tilde\D^{(i)}\right)}\geq\frac{1}{16}\frac{n^*}{n}=\frac{1}{320}\frac{p}{\eps n}.$$
Note also that the $L_2$ norm of the sum of entries of such $\tilde\D^{(i)}$ is always between $n^*-1$ and $n^*+1$.

Summing up, we have shown that for every $n$ and every $\eps>0$, there is a number $M=\Omega\left(\min(n, p/\eps)\right)$ such that for every $\eps$-differentially private algorithm $\A$, there exists a dataset $\D$ of size $n$ with the property that $\ltwo{\sum_{\ell=1}^n d_{\ell}}\in [M-1, M+1]$ such that, with probability at least $1/2$,
$$\ltwo{\A(\D)-q(\D)}\geq \frac{1}{16}\min\left(1,\frac{p}{20\eps n}\right)=\Omega\left(\min\left(1, \frac{p}{\eps n}\right)\right).$$
\end{proof}

\subsection{Proof of Part~2}
We restate below Part~2 of the lemma.

\vspace{0.5cm}

\boldpar{Lemma~\ref{lem:1-way-marg-lower}- Part~2} \textit{  Let $n, p\in\mathbb{N}$, $\epsilon>0$, and $\delta=o(\frac{1}{n})$. There is a number $M=\Omega\left(\min\left( n, \sqrt{p}/\epsilon\right)\right)$ such that for every $(\eps,\delta)$-differentially private algorithm $\A$, there is a dataset $\D=\{d_1, \ldots, d_n\} \subseteq \hypcnz^p$ with ${\small\ltwo{\sum_{i=1}^nd_i}\in[M-1, M+1]}$ such that, with probability at least $1/3$ (taken over the algorithm random coins), we have}
$$\ltwo{\A(\D)-q(\D)}=\Omega\left(\min\left(1, \frac{\sqrt{p}}{\epsilon n}\right)\right)$$
\textit{where $q(\D)=\frac{1}{n}\sum_{i=1}^{n}d_i$.}

\begin{proof}
Let $n\in\mathbb{N}$. Fix $\eps>0$ and $\delta=o(\frac{1}{n})$. Let $\A$ be any $(\eps, \delta)$-differentially private algorithm for answering $q$. Corollary~3.6 of \cite{BUV13} (together with Lemma~2.5 of the same reference) shows that there exists $n^*=\Omega(\frac{\sqrt{p}}{\eps})$ such that for every $n\leq n^*$, there exists a dataset $\D\subset\hypcnz^p$ of size n such that, with probability at least $1/3$, $\ltwo{\A(\D)-q(\D)}>\frac{2}{27}$. Note that to reach this statement from Corollary~3.6 of \cite{BUV13}, first, we have to translate their definition of $(\alpha, \beta)$-accuracy (given by Definition~2.2 in the same reference) to what this implies about the $L_2$-error which is fairly straightforward. Also, note that in their construction, the dataset entries are drawn from $\{0,1\}^p$ whereas here the entries come from $\hypcnz^p$, hence, we need to take this normalization into account when we translate their statement to our setting. Moreover, by a careful inspection of the construction in \cite{BUV13}\footnote{Their construction is based on robust $(n, p)$-fingerprinting codes (see Definition~3.3 in the same reference)}, one can show that such dataset whose existence was argued above has the property that the $L_2$-norm of the sum of its entries is $\Omega(n)$. (See Section~6 of \cite{BUV13} for details). This gives us the desired $\Omega(1)$ lower bound on the $L_2$-error when $n\leq n^*=\Omega(\frac{\sqrt{p}}{\eps})$.

To complete the proof, we need to consider the case where $n>n^*$. To do this, we follow the same argument of the second half of the proof of Part~1 above (the pure $\eps$ case). Namely, we proceed as folllows. First, let $\A$ be any $(\eps, \delta)$-differentially private algorithm (where $\delta=o(\frac{1}{n})$). We construct datasets $\tilde\D$ of size $n$ whose $n^*$ first entries are constructed in the same way datasets $\D$ in the first half of this proof were constructed, and the remaining $n-n^*$ entries are constructed such that half of them contain $\mathbf{c}$ and the other half conatin $-\mathbf{c}$ for some fixed $\mathbf{c}\in\hypcnz^p$. Then, we show that if for all such $\tilde\D$, with probability at least $2/3$, $\ltwo{\A(\tilde\D)-q(\tilde\D)}< \frac{2}{27}\frac{n^*}{n}$, then there is an $(\eps, \delta)$-differentially private algorithm $\hat\A$ such that for all $\D$ of size $n^*$ constructed earlier, we get $\ltwo{\hat\A({\D})-q(\D)}<\frac{2}{27}$ which contradicts the result of the first half of this proof. Hence, there is a dataset $\tilde\D$ of size $n$ (constructed as above) such that, with probability at least $1/3$, $\ltwo{\A(\tilde\D)-q(\tilde\D)}\geq \frac{2}{27}\frac{n^*}{n}=\Omega(\frac{\sqrt{p}}{\eps n})$. Note also that such $\tilde\D$ has the property that the $L_2$ norm of the sum of its entries lies in $[M-1, M+1]$ where $M$ is the $L_2$ norm of the sum of the first $n^*$ entries of $\tilde\D$ which is $\Omega(n^*)=\Omega(\frac{\sqrt{p}}{\eps})$ (following the argument of the first half of this proof).
\end{proof}

\section{Converting Excess Risk Bounds in Expectation to High-probability Bounds}\label{sec:expectatioHigh}

In this paper all of our utility guarantees are in terms of the expectation over the randomness of the algorithm. Although all the utility analysis except for the gradient descent based algorithm (Algorithm \ref{Algo:GradDesc}) provide high-probability guarantees directly, in this section we provide a generic approach for obtaining high-probability guarantee based on the expected risk bounds. The idea is to run the underlying differentially private algorithm $k$-times, with the privacy parameters $\epsilon/k$ and $\delta/k$ for each run. Let $\privtheta_1,\cdots\privtheta_k$ be the vectors output by the $k$-runs. First notice that the vector $\privtheta_1,\cdots,\privtheta_k$ is $(\epsilon,\delta)$-differentially private. Moreover if the algorithm has expected excess risk of $F(\epsilon,\delta)$ (where $F$ is the specific excess risk function of $\epsilon$ and $\delta$), then by Markov's inequality there exist an execution of the algorithm $i\in[k]$ for which the excess risk is $2F(\epsilon/k,\delta/k)$ with probability at least $1-1/2^k$.

One can now use the exponential mechanism from Algorithm \ref{Alg:GenSamp}, to pick the best $\privtheta_i$ from the list. By the same analysis of Theorem \ref{thm:abcd2}, one can show that with probability at least $1-\rho/2$, the exponential mechanism will output a vector $\privtheta$ that has excess risk of $\max\limits_i {\sf{Excess\_risk}}(\privtheta_i)-O\left(\frac{L\ltwo{\C}}{\epsilon}\log(k/\rho)\right)$. Setting $k=\log(2/\rho)$, we have that with probability at lest $1-\rho$, the excess risk for $\privtheta$ is at most $O(F(\frac{\epsilon}{\log(1/\rho)},\frac{\delta}{\log(1/\rho)}))$. Placing this bound in context of the paper, the high probability bounds are only a $\text{poly}\log(1/\rho)$ factor off from the expectation bounds.

\section{Excess Risk Bounds for Smooth Functions}
\label{sec:SmoothFn}

In this section we present the scenario where each of the loss function $\ell(\theta;d)$ (for all $d$ in the domain) is $\beta$-smooth in addition to being $L$-Lipschitz (for $\theta\in\C$). It turns out that both for $\epsilon$ and $(\epsilon,\delta)$-differential privacy, objective perturbation algorithm (see \eqref{eq:a1Obj}) \citep{CMS11,KST12} achieves the best possible error guarantees, where the random variable $b$ is either sampled i) from the Gamma distribution with the kernel $\propto e^{-\frac{\epsilon \ltwo{b}}{2L}}$. or ii) from the Normal distribution  $\mathcal{N}\left(0,\I_p\frac{8L^2\log(1/\delta)}{\epsilon^2}\right)$. In terms of privacy, when the noise vector $b$ is from Gamma distribution, the algorithm is $\epsilon$-differentially private. And when the noise is from Normal distribution, it is $(\epsilon,\delta)$-differentially private.
For completeness purposes, we also state the error bounds from \cite{KST12} (translated to the context of this paper).

\begin{equation}
\privtheta=\arg\min\limits_{\theta\in\C}\empL(\theta;\D)+\frac{\Delta}{2}\ltwo{\theta}^2+\ip{b}{\theta}
\label{eq:a1Obj}
\end{equation}

\begin{thm}[Lipschitz and smooth function] The excess risk bounds are as follows:
\begin{enumerate}
\itemsep 1pt
    \item \citep{CMS11} With Gamma density $\nu_1$, setting $\Delta=\Theta\left(\frac{L p}{\epsilon\ltwo{\C}}\right)$ and assuming $\Delta\geq\frac{\beta}{2\epsilon}$, we have\\
        $\E\left[\empL(\privtheta;\D)-\empL(\theta^*;\D)\right]=O\left(\frac{L \ltwo{\C}p}{\epsilon}\right)$.
   \item \citep{KST12} With Gaussian density, setting $\Delta=\Theta\left(\frac{\sqrt{L^2 p\log(1/\delta)}}{\epsilon\ltwo{\C}}\right)$ and assuming $\Delta\geq\frac{\beta}{2\epsilon}$, we have
        $\E\left[\empL(\privtheta;\D)-\empL(\theta^*;\D)\right]=O\left( \frac{L\ltwo{\C}\sqrt{p\ln(1/\delta)}}{\epsilon}\right)$.
\end{enumerate}

\label{thm:generrorz}
\end{thm}

Additionally when the loss function $\ell(\theta;d)$ is $\Delta$-strongly convex (for $\theta\in\C$) for all $d$ in the domain with the condition that $\Delta\geq\frac{\beta}{2\epsilon}$, one can essentially recover the tight error guarantees for the $\epsilon$ and $(\epsilon,\delta)$ case of differential privacy respectively. The main observation is that for the privacy guarantee to be achieved one need not add the additional regularizer. Although not in its explicit form, a variant of this observation appears in the work of \cite{KST12}. We state the error guarantee from \cite[Theorem 31]{KST12} translated to our setting. Notice that unlike Theorem \ref{thm:generrorz}, the error guarantee in Theorem \ref{thm:generror1} does not depend on the diameter of the convex set $\C$.

\begin{thm}[Lipschitz, smooth and strongly convex function] The excess risk bounds are as follows:
\begin{enumerate}
\itemsep 1pt
    \item With Gamma density $\nu_1$, if $\Delta\geq\frac{\beta}{2\epsilon}$, we have
        $\E\left[\empL(\privtheta;\D)-\empL(\theta^*;\D)\right]=O\left(\frac{L^2 p^2}{n\Delta \epsilon^2}\right)$.
   \item  With Gaussian density, if $\Delta\geq\frac{\beta}{2\epsilon}$, we have
        $\E\left[\empL(\privtheta;\D)-\empL(\theta^*;\D)\right]=O\left( \frac{L^2p\ln(1/\delta)}{n\Delta\epsilon^2}\right)$.
\end{enumerate}
\label{thm:generror1}
\end{thm}

\section{From Excess Empirical Risk to Generalization Error}
\label{app:generalization}

\asnote{Update title (and statements?)}
\rbnote{updated the title and revised the statements. Added a theorem for upper bounds for GLM}
In this section, provide a generic tool to interpret our ERM results in the context of generalization error (\emph{true} risk) bounds. For a given distribution $\tau$, let us define {true} risk for a model $\theta\in\C$ as follows.
\begin{equation}
{\sf TrueRisk}(\theta)=\E_{d\sim\tau}\left[\ell(\theta;d)\right].
\label{eq:trueRisk}
\end{equation}
Analogously, we define the excess risk for a given  a given model $\theta$ by $\erisk(\theta)$.
\begin{equation}
\erisk(\theta)=\E_{d\sim\tau}\left[\ell(\theta;d)\right]-\min\limits_{\theta\in\C}\E_{d\sim\tau}\left[\ell(\theta;d)\right].
\label{eq:excessRisk}
\end{equation}
Let $\D$ be a data set of $n$ data samples drawn i.i.d. from the distribution $\tau$. The following theorem from learning theory relates \emph{true} excess risk to excess empirical risk.
\begin{thm}[Section 5.4 from \cite{SSSS}]
Let $\ell$ be $L-$Lipschitz, $\Delta$-strong convex loss function. With probability at least $1-\gamma$ over the randomness of sampling the data set $\D$, the following is true.
$$\erisk(\theta)\leq \sqrt{\frac{2L^2}{n\Delta}}\sqrt{\empL(\theta;\D)-\min\limits_{\theta\in\C}\empL(\theta;\D)}+\frac{4L^2}{\gamma\Delta n}.$$
\label{thm:exRisk}
\end{thm}

Plugging in the utility guarantee for $\privtheta$ from Theorems \ref{thm:abcd4} and \ref{thm:utility_eps_st_cnvx_vol_samp}, and using the expectation to high-probability bound trick from Appendix \ref{sec:expectatioHigh}, we obtain the following.
\rbnote{the argument that shows the \textbf{expected} excess risk can
  be bounded by $\eps +$ Ex.Emp.Risk may be cool to give?}
\asnote{Yes, I think it is good to at least state it.}
\begin{thm}
[Lipschitz and strongly convex functions]{~}
\begin{enumerate}
    \item There exists an $\epsilon$-differentially private algorithm, that outputs $\privtheta$ such that with probability at least $1-\gamma$ over the randomness of sampling the data set $\D$ and the risk minimization algorithm, the following is true:
$$\erisk(\theta)=O\left(\frac{L^2 p\sqrt{\log n\cdot{\sf poly}\log(1/\gamma)}}{\Delta n\epsilon\gamma}\right).$$
    \item  There exists an $(\epsilon,\delta)$-differentially private algorithm, that outputs $\privtheta$ such that with probability at least $1-\gamma$ over the randomness of sampling the data set $\D$ and the risk minimization algorithm, the following is true:
        $$\erisk(\theta)=O\left(\frac{L^2\sqrt{p}\log^2(n/\delta)\cdot{\sf poly}\log(1/\gamma)}{\Delta n\epsilon\gamma}\right).$$
\end{enumerate}
\label{thm:exRiskPriv}
\end{thm}
\rbnote{if we state our guarantees as expectation bounds rather than high prob, we can get rid of the extra log factors in the pure eps case.}
One can use the following regularization trick to get excess risk guarantees for general convex functions. Let $\ell(\theta;d)$ be an $L$-Lipschitz function and let $\hat\ell(\theta;d)=\ell(\theta;d)+\frac{\Delta}{2}\ltwo{\theta}^2$. Notice that over the convex set $\C$, $\hat\ell$ is $(L+\ltwo{\C})$-Lipschitz and $\Delta$-strongly convex. Let $\erisk_\ell(\theta)$ denote the excess risk for the loss function $\ell$. We can observe the following.
\begin{equation}
\erisk_\ell(\theta)\leq\erisk_{\hat\ell}(\theta)+\frac{\Delta}{2}\ltwo{\C}^2.
\label{eq:exCo}
\end{equation}
Combining with Theorem \ref{thm:exRiskPriv}, plugging in \eqref{eq:exCo},  and optimizing for $\Delta$, we have the following:
\begin{thm}
[Lipschitz functions]{~}
\begin{enumerate}
    \item \asnote{Can't we get generalization error $\min(\sqrt{\frac
          p n},\frac{p}{\eps n})$ for the Lipschitz case, just by
        using uniform convergence?}\rbnote{yes, but I think won't work for $(\eps,\delta)$. I think it's easier to use same argument for both as it's done here.}
There exists an $\epsilon$-differentially private algorithm, that outputs $\privtheta$ such that with probability at least $1-\gamma$ over the randomness of sampling the data set $\D$ and the risk minimization algorithm, the following is true:
$$\erisk(\theta)=O\left(\frac{\sqrt p(L+\ltwo{\C})\ltwo{\C}\left(\log n\cdot{\sf poly}\log(1/\gamma)\right)^{1/4}}{\sqrt{ n\epsilon\gamma}}\right).$$
    \item  There exists an $(\epsilon,\delta)$-differentially private algorithm, that outputs $\privtheta$ such that with probability at least $1-\gamma$ over the randomness of sampling the data set $\D$ and the risk minimization algorithm, the following is true:
        $$\erisk(\theta)=O\left(\frac{p^{1/4}(L+\ltwo{\C})\ltwo{\C}\log(n/\delta)\cdot{\sf poly}\log(1/\gamma)}{ \sqrt{n\epsilon\gamma}}\right).$$
\end{enumerate}
\label{thm:exRiskLip}
\end{thm}

\paragraph{Note.} While the dependence on $n$ for our private algorithms in
Theorem \ref{thm:exRiskLip} matches the bounds for the corresponding
non-private algorithms (see\cite{SSSS}), unlike the non-private
counter parts, the private algorithms have an explicit dependence on
$p$. We leave it as an open problem to figure out the right dependence
on $p$ w.r.t. excess risk for private algorithms.\asnote{update this note.}

\paragraph{Generalized Linear Models (GLM).} \rbnote{added a small section for GLM upper bounds}
When the loss function $\ell$ is a generalized linear function, we
obtain better bounds on the true excess risk. The following bounds are actually tight (see the note after the theorem statement.) \asnote{Explain tightness briefly?}

\begin{thm}
Suppose the loss function $\ell(\theta; d)$ can be written as $g\left(\langle \theta, d\rangle; d\right)$ where $g$ is $L_g$-Lipschitz in its first input and $d\in\mathcal{X}$ where $\ltwo{\mathcal{X}}\leq R$. Let $L=L_g R$.
\begin{enumerate}
    \item  There exists an $\epsilon$-differentially private algorithm, that outputs $\privtheta$ such that with probability at least $1-\gamma$ over the randomness of sampling the data set $\D$ and the risk minimization algorithm, the following is true assuming $n=\omega\left((p/\epsilon)^2\right)$:
$$\erisk(\theta)=O\left(\frac{L\ltwo{C}\sqrt{\log(1/\gamma)}}{\sqrt{n}}\right).$$
    \item  There exists an $(\epsilon,\delta)$-differentially private algorithm, that outputs $\privtheta$ such that with probability at least $1-\gamma$ over the randomness of sampling the data set $\D$ and the risk minimization algorithm, the following is true assuming $n=\omega\left(p/\epsilon^2\right)$:
        $$\erisk(\theta)=O\left(\frac{L\ltwo{C}\sqrt{\log(1/\gamma)}\log^2(n/\delta)}{\sqrt{n}}\right).$$
\end{enumerate}
\label{thm:exRiskGLM}
\end{thm}

\rbnote{I added a brief explanation of the tightness of these bounds.}
The above theorem follows from Theorem~2 in \cite{SSSS} and the
regularization trick above. Theorem~\ref{thm:exRiskGLM} shows that in
case of GLM, we can essentially attain the \emph{non-private} upper
bound of $O(\frac{L\ltwo{C}}{\sqrt{n}})$ which is known to be tight:
for example, if we consider a linear loss function then using a
standard Central Limit Theorem argument (or using standard lower bounds on the minimax error in parametric estimation), one can show that the there exists a distribution on $\mathcal{X}$ for which the true excess risk is $\Omega\left(\frac{L\ltwo{C}}{\sqrt{n}}\right)$.

\rbnote{added small section on using the stability guarantee of dp in obtaining expec. ex risk guarantees for Lipschitz functions.}
For general Lipschitz loss functions, we provide a tool that can be used to obtain expectation (over algorithm's random coins) guarantees on the excess risk (as opposed to the high probability guarantees given in Theorem~\ref{thm:exRiskLip} above.)

For any $L$-Lipschitz loss function $\ell$ and any distribution $\tau$ from which the data points $d_1,..., d_n$ comprising the dataset $\D$ are drawn in i.i.d. fashion, we define 
$$\emprisk(\theta)\triangleq\frac{1}{n}\E_{d_1,...,d_n\sim \tau}\left[\eL(\theta; \D)-\min\limits_{\theta'\in\C}\eL(\theta'; \D)\right].$$

Now, we give the following useful lemma. 
\begin{lem}
\begin{enumerate}
    \item  Let $\privtheta$ denote the output of an $(\eps, 0)$-differentially algorithm. We have 
    $$\E\left[\erisk(\privtheta)\right]=O\left(L\ltwo{\C}~\eps+\E\left[\emprisk(\privtheta)\right]\right).$$
    \item  Let $\privtheta$ denote the output of an $(\eps, \delta)$-differentially algorithm. We have
    $$\E\left[\erisk(\privtheta)\right]=O\left(L\ltwo{\C}~\eps+\ltwo{C}^2~\delta+\E\left[\emprisk(\privtheta)\right]\right).$$
\end{enumerate}
where the expectation in both cases is over the random coins of the algorithm.
\label{lem:priv-stabl-erm-to-exrisk}
\end{lem}

We can use this lemma together with our ERM upper bounds for general Lipschitz functions to give the following expectation guarantees on the excess risk: 
\begin{thm}
[Lipschitz functions: Expectation guarantees]{~}
\begin{enumerate}
    \item There is an $\left(\Theta\left(\sqrt{\frac{p}{n}}\right),~0\right)$-differentially private algorithm, that outputs $\privtheta$ such that the following is true:
$$\E\left[\erisk(\privtheta)\right]=O\left(L\ltwo{C}\sqrt{\frac{p}{n}}\right).$$
    \item  There exists an $\left(\Theta\left(\frac{p^{1/4}\log(n/\delta)}{\sqrt{n}}\right),\delta\right)$-differentially private algorithm, that outputs $\privtheta$ such that the following is true:
        $$\E\left[\erisk(\privtheta)\right]=O\left(L\ltwo{C}\frac{p^{1/4}\log(n/\delta)}{\sqrt{n}}\right).$$
\end{enumerate}
where the expectation in both cases is over the random coins of the algorithm.
\label{thm:expec-exRisk-lip}
\end{thm}


\end{document}